\documentclass[11pt]{article}

\date{}

\usepackage[utf8]{inputenc} 
\usepackage[T1]{fontenc}    
\usepackage{hyperref}       
\usepackage{url}            
\usepackage{booktabs}       
\usepackage{amsfonts}       
\usepackage{nicefrac}       
\usepackage{microtype}      
\usepackage{xspace}
\usepackage{natbib}

\usepackage{comment}
\usepackage{amsmath}
\usepackage{amsthm}
\usepackage{graphicx}
\usepackage{rotating}

\usepackage{amssymb}
\usepackage{mathtools}
\usepackage{xcolor}
\usepackage{bbm}
\usepackage{algorithm}
\usepackage{algorithmic}

\makeatletter

\newcommand\x{\times}

\newcommand{\sig}{\Sigma}

\DeclareMathOperator{\Tr}{Tr}
\newcommand{\ttt}{\tilde{\Theta}}
\newcommand{\tth}{\hat{\Theta}}

\newcommand{\ta}{\tilde{A}}
\newcommand{\tb}{\tilde{B}}
\newcommand{\tc}{\tilde{C}}
\newcommand{\tp}{\tilde{P}}
\newcommand{\tl}{\tilde{L}}
\newcommand{\tk}{\tilde{K}}

\newcommand{\Tw}{T_{w}}

\newcommand{\Alg}{\textsc{\small{LqgOpt}}\xspace}
\newcommand{\Sys}{\textsc{\small{SysId}}\xspace}

\newcommand{\reg}{\textsc{\small{REGRET}}\xspace}
\newcommand{\OO}{\mathcal{O}}

\newcommand{\LQR}{\textsc{\small{LQR}}\xspace}
\newcommand{\LQG}{\textsc{\small{LQG}}\xspace}
\newcommand{\OFU}{\textsc{\small{OFU}}\xspace}

\newcommand{\alg}{\textsc{\small{LqgOpt}}\xspace}

\newcommand{\Gcl}{\mathcal{G}^{cl}}
\newcommand{\Gol}{\mathcal{G}^{ol}}

\newcommand{\ofu}{\textsc{OFU}\xspace}

\newcommand{\D}{\mathcal D}

\newcommand{\R}{\mathbb{R}}

\newtheorem{lemma}{Lemma}[section]
\newtheorem{assumption}{Assumption}[section]
\newtheorem{theorem*}{Theorem}[section]
\newtheorem{definition}{Definition}[section]
\newtheorem{theorem}{Theorem}[section]

\title{Adaptive Control and Regret Minimization in Linear Quadratic Gaussian (LQG) Setting }

\author{%
  Sahin Lale$^1$, Kamyar Azizzadenesheli$^2$, Babak Hassibi$^1$, Anima Anandkumar$^2$\\
  $^1$~Department of Electrical Engineering\\
  $^2$~Department of Computing and Mathematical Sciences\\
  California Institute of Technology, Pasadena\\
  \texttt{\{alale,kazizzad,hassibi,anima\}@caltech.edu}
}

\begin{document}

\maketitle

\begin{abstract}
We study the problem of adaptive control in partially observable linear quadratic Gaussian control systems, where the model dynamics are unknown a priori. We propose \Alg, a novel reinforcement learning algorithm based on the principle of optimism in the face of uncertainty, to effectively minimize the overall control cost. We employ the predictor state evolution representation of the system dynamics and deploy a recently proposed  closed-loop system identification method, estimation, and confidence bound construction. \Alg efficiently explores the system dynamics, estimates the model parameters up to their confidence interval, and deploys the controller of the most optimistic model for further exploration and exploitation. We provide stability guarantees for \Alg, and prove the regret upper bound of $\tilde{\mathcal{O}}(\sqrt{T})$ for adaptive control of linear quadratic Gaussian (\LQG) systems, where $T$ is the time horizon of the problem.
\end{abstract}

\section{Introduction}\label{Introduction}
One of the core challenges in the field of control theory and reinforcement learning is adaptive control. It is the problem of controlling dynamical systems when the dynamics of the systems are unknown to the decision-making agents. In adaptive control, agents interact with given systems in order to explore and control them while the long-term objective is to minimize the overall average associated costs. The agent has to balance between \textit{exploration} and \textit{exploitation}, learn the dynamics, strategize for further exploration, and exploit the estimation to minimize the overall costs. The sequential nature of agent-system interaction results in challenges in the system identification, estimation, and control under uncertainty, and these challenges are magnified when the systems are partially observable, \textit{i.e.} contain hidden underlying dynamics. 

In the linear systems, when the underlying dynamics are fully observable, the asymptotic optimality of estimation methods has been the topic of study in the last decades~\citep{lai1982least,lai1987asymptotically}. Recently, novel techniques and learning algorithms have been developed to study the finite-time behavior of adaptive control algorithms and shed light on the design of optimal methods~\citep{pena2009self,fiechter1997pac, abbasi2011regret}. In particular, \citet{abbasi2011regret} proposes to use the principle of optimism in the face of uncertainty (\OFU) to balance exploration and exploitation in \LQR, where the state of the system is observable. \OFU principle suggests to estimate the model parameters up to their confidence interval, and then act according to the policy advised by the model in confidence set with the lowest optimal cost, known as the optimistic model.  

When the underlying dynamics of linear systems are partially observable, estimating the systems' dynamics requires considering and analyzing unobservable events, resulting in a series of significant challenges in learning and controlling the partially observable systems. A line of prior works are dedicated to the problem of open-loop model estimation~\citet{oymak2018non,sarkar2019finite,tsiamis2019finite} where the proposed methods highly rely on random excitation, uncorrelated Gaussian noise, and do not allow feedback control. Recently, \citet{lale2020logarithmic} introduced a novel estimation method for the general cases of both closed and open-loop identification of linear dynamical systems with unobserved hidden states. Their method provides the first finite-time estimation analysis and construction of confidence sets in partially observable linear dynamical systems when the data is collected using a feedback controller, \textit{i.e.} closed-loop system identification. 

In general, computing the optimal controller requires inferring the latent state of the system, given the history of observations. When the model dynamics are not known precisely, the uncertainties in the system estimation result in inaccurate latent state estimation and inaccurate linear controller. The possibility of accumulation of these errors creates a challenging problem in adaptive control of partially observable linear systems. Therefore, we need to consider these challenges in designing an algorithm that performs desirably. In this work, we employ \textit{regret}, a metric in quantifying the performance of learning algorithms that measures the difference between the cost encountered by an adaptive control agent and that of an optimal controller, knowing the underlying system~\citep{lai1985asymptotically}.

\textbf{Contributions:} In this work, we study the adaptive control of partially observable linear systems from both model estimation/system identification and the controller synthesis perspective. 
We propose \Alg, an adaptive control algorithm for learning and controlling unknown partially observable linear systems with quadratic cost and Gaussian disturbances, i.e., linear quadratic Gaussian (\LQG), for which optimal control exists and has a closed form~\citep{bertsekas1995dynamic}. \Alg interacts with the system, collects samples, estimates the model parameters, and adapts accordingly. \Alg deploys \OFU principle to balance the \textit{exploration} vs. \textit{exploitation} trade-off. Using the predictor form of the state-space equations of the partially observable linear systems, we deploy the least-squares estimation problem introduced in \citet{lale2020logarithmic} and obtain confidence sets on the system parameters. \Alg then uses these confidence sets to find the optimistic model and use the optimal controller for the chosen model for further exploration-exploitation. To analyze the finite-time regret of \Alg, we first provide a stability analysis for the sequence of optimistic controllers. Finally, we prove that \Alg achieves a regret upper bound of $\tilde{\mathcal{O}}({\sqrt{T}})$ for adaptive control of partially observable linear dynamical systems with \textit{convex} quadratic cost, an improvement to the $\tilde{\mathcal{O}}({{T}^{2/3}})$ regret upper bound in the prior work \citet{lale2020regret}, where $T$ is the number of total interactions. 

\citet{simchowitz2020improper} and \citet{lale2020logarithmic} propose algorithms which achieve $\tilde{\mathcal{O}}({\sqrt{T}})$ and $\mathcal{O}(\text{polylog}(T))$ regret bound in partially observable linear systems respectively, under different problem setups. \citet{simchowitz2020improper} employ the theory of online learning, and propose to start with an initial phase of pure exploration, long enough for accurate model estimation. Then this phase is followed by committing to the learned model and deploying online learning for the policy updates. They consider various settings (adversarial and stochastic noise) and attain a similar order regret bound for \textit{strongly convex} cost function. \citet{lale2020logarithmic} provide a new closed-loop system identification algorithm and similarly adopt online learning tools to achieve the first logarithmic regret in partially observable linear dynamical systems with \textit{strongly convex} cost function with stochastic disturbances. These two works heavily rely on the strong convexity whereas the results in this paper considers the more general setting of convex cost function (Table \ref{table:1}). 

\begin{table}
\centering
\caption{Comparison with prior works in partially observable linear dynamical systems.}
 \begin{tabular}{l l l  l } 
 \toprule
 \textbf{Work} & \textbf{Regret} &  \textbf{Cost} &
 \textbf{Identification}   \\ 
 \midrule
 \citet{mania2019certainty} & $\sqrt{T}$ & Strongly Convex & Open-Loop \\
  \citet{simchowitz2020improper} & $\sqrt{T}$ & Strongly Convex & Open-Loop \\
 \citet{lale2020logarithmic} & polylog$(T)$ & Strongly Convex & Closed-Loop  \\
  \midrule
 \citet{lale2020regret} & $T^{2/3}$ & Convex & Open-Loop  \\ 
  \citet{simchowitz2020improper} & $T^{2/3}$ & Convex & Open-Loop  \\
 \textbf{This Work} & $\sqrt{T}$ & Convex & Closed-Loop 
 \\
 \bottomrule
\end{tabular}
\label{table:1}
\end{table}

\section{Preliminaries}\label{prelim}
We denote the Euclidean norm of a vector $x$ as $\|x\|_2$.
We denote 
$\rho(A)$ as the spectral radius of a matrix $A$, $\| A\|_F$ as its Frobenius norm and $\| A \|_2$ as its spectral norm. $\Tr(A)$ is its trace, $A^\top$ is the transpose, $A^{\dagger}$ is the Moore-Penrose inverse. The $j$-th singular value of a rank-$n$ matrix $A$ is denoted by $\sigma_j(A)$, where
$\sigma_{\max}(A):=\sigma_1(A) \geq \sigma_2(A) \geq \ldots \geq \sigma_{\min}(A):=\sigma_n(A) > 0$. $I$ represents the identity matrix with the appropriate dimensions. 

Consider the following discrete time linear time-invariant system  $\Theta = (A, B, C)$ and with dynamics as:
\begin{align}
    x_{t+1}& = A x_t + B u_t + w_t \nonumber \\
    y_t& = C x_t + z_t \label{output}.
\end{align}
At each time step $t$, the system is at (hidden) state $x_t \in \mathbb{R}^{n}$, the agent receives observation $y_t \in \mathbb{R}^{m}$ under a measurement noise $z_{t} \sim \mathcal{N}\left(0, \sigma_{z}^{2} I\right)$. Then the agent applies a control input $u_t \in \mathbb{R}^{p}$, and receives a cost of $c_t=y_{t}^{\top} Q y_{t}+u_{t}^{\top} R u_{t}$ where $Q$ and $R$ are positive semidefinite and positive definite matrices, respectively. After taking $u_t$, the state of the system evolves to $x_{t+1}$ for the time step $t+1$ under a process noise $w_{t} \sim \mathcal{N}\left(0, \sigma_{w}^{2} I\right)$. Here the noises are i.i.d. random vectors and $\mathcal{N}(\mu, \Sigma)$ denotes a multivariate normal distribution with mean vector $\mu$ and covariance matrix $\Sigma$. 

\begin{definition}\label{c_o_def}
A linear system $\Theta = (A,B,C)$ is $(A, B)$ controllable if the controllability matrix, 
\begin{equation*}
   \mathbf{C}(A,B,n)=[B \enskip AB \enskip A^2B \ldots A^{n-1}B] 
\end{equation*}
has full row rank. For all $H\geq n$, $\mathbf{C}(A,B,H)$ defines the extended $(A, B)$ controllability matrix. Similarly, a linear system $\Theta = (A,B,C)$ is $A,C$ observable if the observability matrix,
\begin{equation*}
   \mathbf{O}(A,C,n)=[C^\top \enskip (CA)^\top \enskip (CA^2)^\top \ldots (CA^{n-1})^\top]^\top
\end{equation*}
has full column rank. For all $H\geq n$, $\mathbf{O}(A,C,H)$ defines the extended $(A, C)$ observability matrix.
\end{definition}

Suppose the underlying system is controllable and observable. Then, the agent chooses control inputs as a function of past observations and aims to minimize the expected cost,
\begin{equation*} 
J_{\star}(\Theta) \!=\!  \lim _{T \rightarrow \infty} \min _{u=[u_{1}, \ldots, u_{T}]} \frac{1}{T} \mathbb{E}\left[\sum_{t=1}^{T} y_{t}^{\top} Q y_{t}+u_{t}^{\top} R u_{t}\right].
\end{equation*}
This problem is known as \LQG control. The optimal solution to \LQG control problem is a linear feedback control policy given as $u_t = -K\hat{x}_{t|t,\Theta}$. Here $K$ is the optimal feedback gain matrix,
\begin{align*}
    K = \left(R+B^{\top} P B\right)^{-1} B^{\top} P A,
\end{align*}
where $P$ is the unique positive semidefinite solution to the following discrete-time algebraic Riccati equation (DARE):
\begin{equation} \label{DARE1}
    P = A^\top  P  A  +  C^\top  Q C  -  A^\top  P B  \left(  R  +  B^\top  P B \right)^{-1}  B^\top  P  A,  
\end{equation}
and $\hat{x}_{t|t,\Theta}$ is the minimum mean square error (MMSE) estimate of the underlying state using system parameters $\Theta$ and past observations, where $\hat{x}_{0|-1,\Theta}=0$. At steady-state, this estimate is efficiently obtained by using the Kalman filter:
\begin{align}
     &\hat{x}_{t|t,\Theta} = \left( I - LC\right)\hat{x}_{t|t-1,\Theta} + Ly_t, \label{estimation_wt_y} \\ &\hat{x}_{t|t-1,\Theta} = (A\hat{x}_{t-1|t-1,\Theta} + Bu_{t-1}), \label{estimation_wo_y} \\
     &L = \sig  C^\top  \left( C \sig  C^\top + \sigma_z^2 I \right)^{-1}, \label{kalman_gain}
\end{align}
where $\sig $ is the unique positive semidefinite solution to the following DARE: 
\begin{equation*}
    \sig = A \sig   A^\top - A \sig  C^\top \left( C \sig  C^\top + \sigma_z^2 I \right)^{-1} C \sig  A^\top + \sigma_w^2 I.
\end{equation*}
In the adaptive control, the underlying system parameters $\Theta$ are unknown, and the agent needs to learn them through interaction with the system with the aim of minimizing the cumulative costs $\sum_{t=1}^T c_t$ after $T$ time steps. We measure the performance of the agent using regret, \textit{i.e.,} the difference between the agent's cost and the optimal expected cost:
\begin{equation*}
    \reg(T) = \sum_{t=0}^T (c_t - J_*(\Theta)).
\end{equation*}

The system characterization depicted in (\ref{output}) is called state-space form of the system $\Theta$. The same discrete time linear time-invariant system can be represented in several ways which has been considered in various works in control theory and reinforcement learning \citep{kailath2000linear,tsiamis2019sample,lale2020logarithmic}. Note that these representations all have the same second order statistics. One of the most common form is the innovations form\footnote{For simplicity, all of the system representations are presented for the steady-state of the system.} of the system characterized as
\begin{align}
    x_{t+1}&=A x_t+B u_t+F e_t \nonumber \\
    y_t&=C x_t+e_t \label{innovation}.
\end{align}
where $F = A L$ is the Kalman gain in the observer form and $e_t$ is the zero mean white innovation process. In this equivalent representation of system, the state $x_t$ can be seen as the estimate of the state in the state space representation, which is the expression stated in (\ref{estimation_wo_y}). In the steady state, $e(t) \sim \mathcal{N}\left(0,C \sig  C^\top + \sigma_z^2 I \right)$. Using the relationship between $e_t$ and $y_t$, we obtain the following characterization of the system $\Theta$, known as the predictor form of the system,
\begin{align}
    x_{t+1}&=\Bar{A} x_t+B u_t+F y_t \nonumber \\
    y_{t}&=C x_t+e_t \label{predictor}
\end{align}
where $\Bar{A} = A - F C$ and $ F = AL$.
Notice that at steady state, the predictor form allows the current output $y_t$ to be described by the history of inputs and outputs with an i.i.d. Gaussian disturbance $e_t \sim \mathcal{N}\left(0,C \sig  C^\top + \sigma_z^2 I \right)$. In this paper, we exploit these fundamental properties to estimate the underlying system, even with feedback control. We consider the set of stable systems.
\begin{assumption} \label{Stable}
The system is order $n$ and
minimal in the sense that the system cannot be described by a state-space model of order less than $n$. The system is stable, \textit{i.e.} $\rho(A) < 1$ and $\Phi(A) := \sup _{\tau \geq 0} \frac{\left\|A^{\tau}\right\|}{\rho(A)^{\tau}} < \infty$.
\end{assumption}
Note that the assumption regarding $\Phi(A)$ is required for quantifying the finite time evolution of the system and it is a mild condition, \textit{e.g.} if $A$ is diagonalizable, $\Phi(A)$ is finite. Additionally for stable $A$, $\Phi(A)$ can be upper bounded by the $\mathcal{H}_\infty$-norm of the system $x_{t+1} = Ax_t + w_t$ \citep{mania2019certainty}.

We assume that the underlying system lives in the following set.

\begin{assumption} \label{AssumContObs}
The unknown system $\Theta =(A,B,C)$ is a member of a set $\mathcal{S}$, such that, 
\begin{equation*}
  \mathcal{S} \subseteq \left\{ \Theta' =  (A' ,B' ,C', F') \;
    \begin{tabular}{|l}
      $\rho(A')<1$, \\
      $(A',B')$ is controllable, \\
      $(A',C')$ is observable, \\
      $(A',F')$ is controllable.
\end{tabular}
\right\}
\end{equation*}
\end{assumption}

The above assumptions are standard in system identification settings in order to ensure the possibility of accurate estimation of the system parameters \citep{knudsen2001consistency, oymak2018non, tsiamis2019finite, sarkar2019finite,tsiamis2019sample,lale2020logarithmic,lale2020regret}. 

\begin{assumption}\label{Stabilizable set}
The set $\mathcal{S}$ consists of systems that are contractible, \textit{i.e.,}
\begin{align*}
    \rho \coloneqq \sup_{\Theta' = (A',B',C') \in \mathcal{S}} \left\|A' - B'K(\Theta') \right\| < 1, 
\end{align*}
where $K(\Theta')$ is the optimal feedback gain matrix of $\Theta'$, and
\begin{align*}
    \upsilon \coloneqq \sup_{\Theta' = (A',B',C') \in \mathcal{S}} \left\|A' - A'L(\Theta')C'\right\| < 1.
\end{align*}
where $L(\Theta')$ is the optimal Kalman gain matrix of $\Theta'$. There exists finite numbers $D,~\Gamma,~\zeta$ such that $D = \sup_{\Theta' \in \mathcal{S}} \|P(\Theta') \|$, $\Gamma = \sup_{\Theta' \in \mathcal{S}} \|K(\Theta') \|$ and $\zeta = \sup_{\Theta' \in \mathcal{S}} \|L(\Theta') \|$. 
\end{assumption}
This assumption allows us to develop stability guarantees in the presence of sub-optimal closed-loop controllers.

\section{Adaptive Control via \Alg}
In this section, we present \Alg, an adaptive control algorithm for \LQG {} control problems, and describe its compounding components. The outline of \Alg is given in Algorithm \ref{algo}. The early stage of deploying \alg involves a fixed warm-up period dedicated for pure exploration using Gaussian excitation. \alg requires this exploration period to estimate the model parameters reliably enough that the controller designed based on the parameter estimation and their confidence set results in a stabilizing controller on the real system. The duration of this period depends on how stabilizable the true parameters are and how accurate the model estimations should be. We formally quantify these statements and the length of the warm-up period. 

After the warm-up period, \Alg utilizes the model parameter estimations and their confidence sets to design a controller corresponding to an optimistic model in the confidence sets, obtained by following the \OFU principle. Due to the reliable estimation from the warm-up period, this controller and all the future designed controller stabilize the underlying true unknown model. The agent deploys the prescribed controller on the real system for exploration and exploitation. The agent collects samples throughout its interaction with the environment, and use these samples for further model estimation, confidence interval construction, and design of the controller regarding to an optimistic model. The agent repeats this process.

Since the Kalman filter converges exponentially fast to the steady-state gain in observer form, without loss of generality, we assume that $x_0 \sim \mathcal{N}(0, \Sigma)$, \textit{i.e.}, the system starts at the steady-state. This consideration eases the presentation of the algorithm. We provide the overview of the analysis for any arbitrary and almost surely finite initialization in the Appendix~\ref{SupSec:NonSteady}.

In the warm-up period \Alg excites the system with $u_t \sim \mathcal{N}(0,\sigma_u^2 I)$ for $1\leq t\leq \Tw$. Considering the predictor form representation of the system given in (\ref{predictor}), we can roll back the state evolution $H$ time steps back as follows,
\begin{equation*}
    x_t = \sum_{k=0}^{H-1} \Bar{A}^k \left(Fy_{t-k-1} \!+\! Bu_{t-k-1} \right) +\Bar{A}^{H} x_{t-H}
\end{equation*}
From Assumption \ref{AssumContObs}, we have that $\Bar{A}$ is stable, thus the state can be estimated in principle for large enough $H$. Using the generated input-output sequence $\D = \{y_t, u_t \}^{\Tw}_{t=1}$, \Alg constructs $N$ subsequences of $H$ input-output pairs, $\phi_t$ for $H \leq t \leq \Tw$, where $\Tw= H + N - 1$,
\begin{equation*}
    \phi_t = \left[ y_{t-1}^\top \ldots y_{t-H}^\top \enskip u_{t-1}^\top \ldots \enskip u_{t-H}^\top \right]^\top \in \mathbb{R}^{(m+p)H}.
\end{equation*}

Using this definition, following \citet{lale2020logarithmic}, we can write the following truncated autoregressive exogenous (ARX) model for the given system $\Theta$,
\begin{equation}
    y_{t} = \mathbf{M} \phi_{t} + e_{t} + C \Bar{A}^H x_{t-H}
\end{equation}
where $\mathbf{M} \in \R^{m \times (m+p)H}$ defined as
\begin{equation}
    \mathbf{M}\!=\! \left[CF,\enskip C\Bar{A}F,\enskip  \ldots,\enskip C \Bar{A}^{H-1}F,\enskip CB,\enskip C\Bar{A}B,\enskip \ldots, \enskip C\Bar{A}^{H-1} B \right].
\end{equation}
Thus, any input-output trajectory $\{y_i, u_t \}^{T}_{t=1}$ can be represented as 
\begin{equation} \label{newform}
    Y_T = \Phi_T \mathbf{M}^{\top} + E_T + N_T
\end{equation}
where 
\begin{align*}
    Y_T &= \left[ y_H,~y_{H+1},~ \ldots,~y_T \right]^\top \in \mathbb{R}^{N \times m} \\
    \Phi_T &= \left[ \phi_H,~\phi_{H+1},~\ldots,~\phi_T \right]^\top \in \mathbb{R}^{N \times (m+p)H} \\
    E_T &= \left[ e_H,~e_{H+1},~\ldots,~e_T \right]^\top \in \mathbb{R}^{N \times m} \\
    N_T\!&=\! \left[C \Bar{A}^H\!x_{0},~C \Bar{A}^H\!x_{1},\ldots,C \Bar{A}^H \!x_{T-H} \right]^\top \!\!\in\! \mathbb{R}^{N \times m}  
\end{align*}
for $N = T-H+1$. 

\begin{algorithm}[t] 
\caption{\Alg}
  \begin{algorithmic}[1]
 \STATE \textbf{Input:} $\Tw$, $H$, $\sigma_{o}$, $\sigma_c$, $S>0$, $\delta > 0$, $n$, $m$, $p$, $Q$, $R$ \\

 ------ \textsc{\small{Warm-Up}} ------------------------------------------------ \\
\FOR{$t = 0, 1, \ldots, \Tw$}
\STATE Deploy $u_t \!\sim\! \mathcal{N}(0,\sigma_u^2 I)$ and store $\mathcal{D}_0 \!=\! \lbrace y_t,u_t \rbrace_{t=1}^{\Tw}$
\ENDFOR \\
------ \textsc{\small{Adaptive Control}} ----------------------------------- \\
\FOR{$i = 0, 1, \ldots$}
    \STATE Calculate $\mathbf{\hat{M}_{i}}$ using $\mathcal{D}_i = \lbrace y_t,u_t \rbrace_{t=1}^{2^i \Tw}$
    \STATE Deploy \Sys($H,\mathbf{\hat{M}_{i}},n$) for $\hat{A}_i, \hat{B}_i, \hat{C}_i, \hat{L}_i$
    \STATE Construct the confidence sets $\mathcal{C}_A(i), \mathcal{C}_B(i), \mathcal{C}_C(i),$ $\mathcal{C}_L(i)$ s.t. w.h.p.
    $(A,B,C,L) \!\in\! \mathcal{C}_i$, where \\ $\mathcal{C}_i\!\! \coloneqq \!\! (\mathcal{C}_A(i) \times \mathcal{C}_B(i) \times \mathcal{C}_C(i) \times \mathcal{C}_L(i))$ 
    \STATE Find a $\ttt_i = (\ta_i,\tb_i,\tc_i,\tl_i ) \in \mathcal{C}_i \cap \mathcal{S} $ s.t.\\
    $ \quad \qquad J(\ttt_i) \leq \inf_{\Theta' \in \mathcal{C}_i \cap \mathcal{S}} J(\Theta') + T^{-1}$
    \FOR{$t = 2^i\Tw, \ldots 2^{i+1}\Tw-1$ }
        \STATE Execute the optimal controller for $\ttt_i$
    \ENDFOR
\ENDFOR
  \end{algorithmic}
 \label{algo} 
\end{algorithm}
Note that, during the warm-up period the noise terms are zero-mean including the effect of initial state since we assume that $x_0 \sim \mathcal{N}(0, \Sigma)$.  After the warm-up period, \Alg obtains the first estimate of the unknown truncated ARX model $\mathbf{M}$ by solving the following regularized least square problem first introduced in \citet{lale2020logarithmic},
\begin{equation}\label{estimation}
    \mathbf{\hat{M}_{0} } = \arg \min_X \| Y_{\Tw} - \Phi_{\Tw} X^{\top}\|^2_F + \lambda \|X\|^2_F 
\end{equation}
where the solution
\begin{align*}
    \mathbf{\hat{M}_{0}}^\top = (\Phi_{\Tw}^\top \Phi_{\Tw} + \lambda I)^{-1} \Phi_{\Tw}^\top Y_{\Tw}.
\end{align*}
Using this solution, \Alg deploys a system-identification algorithm and obtains the estimates of the system parameters $\hat{A}_0, \hat{B}_0, \hat{C}_0, \hat{L}_0$, with corresponding confidence sets $\mathcal{C}_A(0), \mathcal{C}_B(0), \mathcal{C}_C(0),$ $\mathcal{C}_L(0)$ in which the underlying system parameters live with high probability. With the initial confidence sets, \Alg starts adaptive control period using the \OFU principle. It selects the optimistic model \textit{i.e.}, the model  that has the minimum average expected cost, among the plausible models and executes the optimal controller for the chosen model. As the confidence sets shrink, \textit{i.e.}, the estimates of system parameters are \textit{significantly} refined, \Alg adapts and updates its policy by deploying \OFU principle on the new confidence sets. 

For a linear system $\Theta = (A,B,C)$, we define truncated open-loop and closed-loop noise evolution parameters, respectively $\Gol$ and $\Gcl$. When the controller is set to be i.i.d. Gaussian excitements, $\Gol \in \R^{H(m+p) \times 2H(n+m+p)}$ encodes the open-loop evolution of the disturbances in the system, and represents the responses to these disturbances on the \textit{batch} of observations and actions \textit{history}. Note that the historical data is correlated even in the open-loop setting with i.i.d. Gaussian excitements. The exact definition of $\Gol$ is provided in equation (\ref{Gol}) of Appendix \ref{Golsubsection}. In Appendix \ref{Golsubsection}, we also show that $\Gol$ is full row-rank, \textit{i.e.}, $\sigma_{\min}(\Gol) > \sigma_o > 0$, where $\sigma_o$ is known to \Alg. 

When the controller is set to be the optimal policy for the underlying system, \textit{i.e.} closed-loop system, $\Gcl \in \R^{H(m+p) \times 2H(n+m)}$ represents the translation of the truncated history of process and measurement noises on the inputs, $\phi$'s. The exact construction of $\Gcl$ is provided in detail in equation (\ref{Gcl}) of Appendix \ref{Gclsubsection}. Briefly, it is formed by shifting a block matrix $\mathbf{\Bar{G}} \in \R^{(m+p)\times 2H(n+m)}$ by $m+n$ in each block row where $\mathbf{\Bar{G}}$ is constructed by $H$ $(m+p)\times (n+m)$ matrices. We assume that $H$ used in \Alg is large enough that $\mathbf{\Bar{G}}$ is full row rank for the given system. In Appendix \ref{Gclsubsection}, we show that, if $\mathbf{\Bar{G}}$ is full row-rank, $\Gcl$ would be full row-rank, too. Thus, we have that for the choice of $H$ in \Alg, $\sigma_{\min}(\Gcl)$ is lower bounded by some positive value, \textit{i.e.}, $\sigma_{\min}(\Gcl) > \sigma_c > 0$, where \Alg only knows $\sigma_c$ and searches for an optimistic system whose closed-loop noise evolution parameter satisfies this lower bound. 


The following theorem states the main result of the paper, an end-to-end regret upper bound of the adaptive control in \LQG systems.

\begin{theorem}[Regret Upper Bound]\label{total regret main text}
Given a \LQG $\Theta = (A,B,C)$, and regulating parameters $Q$ and $R$, with high probability, the regret of \Alg with a warm-up duration of $\Tw = poly(H, \log(T), \sigma_o, \sigma_c, \upsilon, \zeta, \Gamma, m, n, p, \rho, \Phi(A))$ is 
\begin{equation}
    \reg(T) =  \tilde{\OO}\left( \sqrt{T}\right)
\end{equation}
\end{theorem}
The exact expressions that define $\Tw$ are given in Appendix with the detailed definitions.
\subsection{Learning the Truncated ARX Model}
The following results are adapted from \citet{lale2020logarithmic}. First consider the effect of truncation bias term, $N_t$. From Assumption \ref{Stabilizable set}, we have that $\| \Bar{A}\| \leq \upsilon <1$. Thus, each term in $N_t$ is order of $\upsilon^H$. In order to get consistent estimation, for some problem dependent constant $c_H$, \Alg sets $H \geq \frac{\log(c_H T^2 \sqrt{m} / \sqrt{\lambda})}{\log(1/\upsilon)}$, resulting in a negligible bias term of  order $1/T^2$. The following Theorem 3 of \citet{lale2020logarithmic} gives the self-normalized finite sample estimation error of (\ref{estimation}).

\begin{theorem}[Closed-Loop Identification, \citep{lale2020logarithmic}]
\label{theo:closedloopid}
Let $\mathbf{\hat{M}_t}$ be the solution to (\ref{estimation}) at time $t$. For the given choice of $H$, define 
\begin{align*}
    V_t = V + \sum_{i=H}^{t} \phi_i \phi_i^\top 
\end{align*}
where $V = \lambda I$. Let $\|\mathbf{M}\|_F \leq S$. For $\delta \in (0,1)$, with probability at least $1-\delta$, for all $t$, $\mathbf{M}$ lies in the set $\mathcal{C}_{\mathbf{M}}(t)$, where 
\begin{equation*}
    \mathcal{C}_{\mathbf{M}}(t) = \{ \mathbf{M}': \Tr((\mathbf{\hat{M}_t} - \mathbf{M}')V_t(\mathbf{\hat{M}_t}-\mathbf{M}')^{\top}) \leq \beta_t \},
\end{equation*}
for $\beta_t$ defined as follows,
\begin{equation*}
    \beta_t = \left(\sqrt{m\| C \sig  C^\top + \sigma_z^2 I \| \log  \left(\frac{\operatorname{det}\left(V_t\right)^{1 / 2}}{\delta \operatorname{det}(V)^{1 / 2} }\right)} + S\sqrt{\lambda} +\frac{t\sqrt{H}}{T^2} \right)^2.
\end{equation*}
\end{theorem}

For completeness, the proof is given in Appendix \ref{SupSec:SysId}. It uses self-normalized tail inequalities to get the first two terms in the definition of $\beta_t$ , and with the given choice of $H$, we obtain the final term in the bound. This bound can be translated to $\|\mathbf{\hat{M}_t} - \mathbf{M}\|$ in order to be utilized for the confidence set construction of the system parameters. First, we need the following lemmas that guarantee persistence of excitation during the warm-up period and adaptive control period.

\begin{lemma}[Persistence of Excitation in Warm-Up Period] \label{openloop_persistence}
After sufficient time steps in warm-up period of \Alg, with probability at least $1-\delta$, we have
\begin{equation}
\sigma_{\min}\left(\sum_{i=1}^{t} \phi_i \phi_i^\top \right) \geq t \frac{\sigma_{o}^2 \min \{ \sigma_w^2, \sigma_z^2, \sigma_u^2 \}}{2}.  
\end{equation}
\end{lemma}

\begin{lemma}[Persistence of Excitation in Adaptive Control Period] \label{closedloop_persistence}
After sufficient time steps in adaptive control period of \Alg, with probability $1-3\delta$, we have
\begin{equation}
\sigma_{\min}\left(\sum_{i=1}^{t} \phi_i \phi_i^\top \right) \geq t \frac{\sigma_{c}^2 \min \{ \sigma_w^2, \sigma_z^2\}}{16}. 
\end{equation}
\end{lemma}

For two problem dependent parameters $\Upsilon_w$ and $\Upsilon_c$, that uniformly bound the components of $\phi$'s during the warm-up and adaptive control period respectively, we have the following theorem which combines Theorem \ref{theo:closedloopid} with Lemma \ref{openloop_persistence} and \ref{closedloop_persistence} to obtain the bound over $\|\mathbf{\hat{M}_t} - \mathbf{M}\|$.

\begin{theorem}\label{theo:id2norm}
During the warm-up period, $\|\phi_{t}\| \leq \Upsilon_w \sqrt{H}$ with high probability. After the warm-up period of $\Tw$, the initial estimation of the truncated ARX model, $\mathbf{\hat{M}_0}$, obeys 
\begin{equation*} 
    \|\mathbf{\hat{M}_0} - \mathbf{M}\| \leq  \frac{poly(m,H,p)}{ \min \{ \sigma_w, \sigma_z, \sigma_u \} \sigma_o \sqrt{\Tw}}.
\end{equation*}

During the adaptive control, with high probability $\|\Phi_{t}\| \leq \Upsilon_c \sqrt{H}$. For the adaptive control period at any time $t\geq 2\Tw$, the least squares estimate of the truncated ARX model  $\mathbf{\hat{M}_t}$ follows
\begin{equation*}
    \|\mathbf{\hat{M}_t} - \mathbf{M}\| \leq  \frac{poly(m,H,p)}{\sqrt{t \min\{ \sigma_o^2\sigma_w^2, \sigma_o^2\sigma_z^2, \sigma_o^2\sigma_u^2, \frac{\sigma_c^2\sigma_w^2 }{8}, \frac{\sigma_c^2\sigma_z^2 }{8} \}}}.
\end{equation*}
\end{theorem}

Note that the choice of $H$ depends on the horizon, which is needed to be known apriori. Since the dependency of $H$ in the horizon $T$ is $\log(T)$, one can deploy the standard doubling trick to relax this requirement.\footnote{Doubling trick suggests to set the horizon to a time step, and in a repeated fashion, whenever that time step is reached, double that time step, and continue.}

\subsection{System Identification}
After estimating $\mathbf{\hat{M}_{t}}$, \Alg constructs confidence sets for the unknown system parameters and uses these confidence sets to come up with the optimistic controller to exploit the information gathered. \Alg uses a subspace identification algorithm \Sys introduced by \citet{lale2020logarithmic} and for completeness given in Algorithm \ref{SYSID} in Appendix \ref{SupSec:ConfSet}. \Sys is similar to Ho-Kalman method~\citep{ho1966effective} and estimates the system parameters from $\mathbf{\hat{M}_{t}}$. First of all, notice that $\mathbf{M} = [\mathbf{F}, \mathbf{G}]$ where 
\begin{align*}
    \mathbf{F} &=  \left[CF, \enskip C\Bar{A}F, \enskip \ldots, \enskip C \Bar{A}^{H-1} F \right] \in \mathbb{R}^{m \times mH}, \\
    \mathbf{G} &= \left[ CB, \!\enskip C\Bar{A}B, \enskip \ldots, \enskip C \Bar{A}^{H-1} B \right] \in \mathbb{R}^{m \times pH}.
\end{align*}
Given the estimate for the truncated ARX model
\begin{align*}
    \mathbf{\hat{M}_{t}} = [\mathbf{\hat{F}_{t,1}},\ldots, \mathbf{\hat{F}_{t,H}}, \mathbf{\hat{G}_{t,1}},\ldots, \mathbf{\hat{G}_{t,H}}],
\end{align*}
where $\mathbf{\hat{F}_{t,i}}$ is the $i$'th $m \times m$ block of $\mathbf{\hat{F}_{t}}$, and $\mathbf{\hat{G}_{t,i}}$ is the $i$'th $m \times p$ block of $\mathbf{\hat{G}_{t}}$ for all $1 \leq i \leq H$, \Sys constructs two $d_1 \times (d_2+1)$ Hankel matrices $\mathbf{\mathcal{H}_{\hat{F}_t}}$ and $\mathbf{\mathcal{H}_{\hat{G}_t}}$ such that $(i,j)$th block of Hankel matrix is $\mathbf{\hat{F}_{t,i}}$ and $\mathbf{\hat{G}_{t,i}}$ respectively. Then, it forms the following matrix $\hat{\mathcal{H}}_t$.
\begin{align*}
    \hat{\mathcal{H}}_t = \left[ \mathbf{\mathcal{H}_{\hat{F}_t}}, \enskip \mathbf{\mathcal{H}_{\hat{G}_t}} \right].
\end{align*}
Recall that the dimension of latent state, $n$, is the order of the system for the observable and controllable system. For some problem dependent constant $c_H$, let $H \geq \max \left\{2n+1, \frac{\log(c_H T^2 \sqrt{m} / \sqrt{\lambda})}{\log(1/\upsilon)} \right\}$, we can pick $d_1 \geq n$ and $d_2 \geq n$ such $d_1+d_2+1 = H$. This guarantees that the system identification problem is well-conditioned. Using  Definition \ref{c_o_def}, if the input to the \Sys was $\mathbf{M} = [\mathbf{F}, \mathbf{G}]$ then constructed Hankel matrix, $\mathcal{H}$ would be rank $n$, 
\begin{align*}
     \mathcal{H} &=  [ C^\top, ~\ldots, \enskip (C\Bar{A}^{d_1-1}) ^\top] ^\top [F,~\ldots,~ \Bar{A}^{d_2}F,~B,~\ldots,~\Bar{A}^{d_2}B]\\
     &=\mathbf{O}(\bar{A},C,d_1) \enskip [\mathbf{C}(\bar{A},F,d_2+1),\quad \Bar{A}^{d_2}F, \quad \mathbf{C}(\bar{A},B,d_2+1), \quad \Bar{A}^{d_2}B] \\
     &= \mathbf{O}(\bar{A},C,d_1) \enskip [F,\quad \Bar{A}\mathbf{C}(\bar{A},F,d_2+1), \quad B, \quad \Bar{A}\mathbf{C}(\bar{A},B,d_2+1)].
\end{align*}
Notice that $\mathbf{M}$ and $\mathcal{H}$ are uniquely identifiable for a given system $\Theta$, whereas for any invertible $\mathbf{T}\in \R^{n \times n}$, the system resulting from
\begin{align*}
    A' = \mathbf{T}^{-1}A\mathbf{T},~ B' = \mathbf{T}^{-1}B,~ C' = C \mathbf{T},~ L' = \mathbf{T}^{-1}L 
\end{align*}
gives the same $\mathbf{M}$ and $\mathcal{H}$. Similar to Ho-Kalman algorithm, \Sys computes the SVD of $\mathbf{\hat{M}_{t}}$ and estimates the extended observability and controllability matrices and eventually system parameters up to similarity transformation. To this end, \Sys constructs $\hat{\mathcal{H}}_t^-$ by discarding $(d_2 + 1)$th and $(2d_2 + 2)$th block columns of $\hat{\mathcal{H}}_t$, \textit{i.e.} if it was $\mathcal{H}$ then we have,
\begin{align*}
    \mathcal{H}^- = \mathbf{O}(\bar{A},C,d_1) \enskip [\mathbf{C}(\bar{A},F,d_2+1), \quad \mathbf{C}(\bar{A},B,d_2+1)].
\end{align*}
The algorithm then calculates $\hat{\mathcal{N}}_t$, the best rank-$n$ approximation of $\hat{\mathcal{H}}_t^-$, obtained by setting its all but top $n$ singular values to zero. The estimates of $\mathbf{O}(\bar{A},C,d_1)$,  $\mathbf{C}(\bar{A},F,d_2+1)$ and $\mathbf{C}(\bar{A},B,d_2+1)$ are given as 
\begin{equation*}
    \hat{\mathcal{N}}_t = \mathbf{U_t} \mathbf{\Sigma_t}^{1/2}~ \mathbf{\Sigma_t}^{1/2}\mathbf{V_t}^\top =  \mathbf{\hat{O}_t}(\bar{A},C,d_1) \enskip [\mathbf{\hat{C}_t}(\bar{A},F,d_2+1), \quad \mathbf{\hat{C}_t}(\bar{A},B,d_2+1)].
\end{equation*}

From these estimates \Sys recovers $\hat{C}_t$ as the first $m \times n$ block of $\mathbf{\hat{O}_t}(\bar{A},C,d_1)$, $\hat{B}_t$ as the first $n \times p$ block of $\mathbf{\hat{C}_t}(\bar{A},B,d_2+1)$ and $\hat{F}_t$ as the first $n \times m$ block of $\mathbf{\hat{C}_t}(\bar{A},F,d_2+1)$. Let $\hat{\mathcal{H}}_t^+$ be the matrix obtained by discarding $1$st and $(d_2+2)$th block columns of $\hat{\mathcal{H}}_t$, \textit{i.e.} if it was $\mathcal{H}$ then
\begin{align*}
    \mathcal{H}^+ = \mathbf{O}(\bar{A},C,d_1) \enskip \Bar{A} \enskip [\mathbf{\hat{C}_t}(\bar{A},F,d_2+1), \quad \mathbf{\hat{C}_t}(\bar{A},B,d_2+1)].
\end{align*}
Therefore, \Sys recovers
\begin{align*}
    \hat{\Bar{A}}_t = \mathbf{\hat{O}_t}^\dagger(\bar{A},C,d_1) \enskip \hat{\mathcal{H}}_t^+ \enskip [\mathbf{\hat{C}_t}(\bar{A},F,d_2+1), \quad \mathbf{\hat{C}_t}(\bar{A},B,d_2+1)]^\dagger.
\end{align*}
Using the definition of $\Bar{A} = A - FC$, the algorithm obtains $\hat{A}_t = \hat{\Bar{A}}_t + \hat{F}_t \hat{C}_t$. Recall that $F = AL$. Using the Assumption \ref{AssumContObs}, \Sys finally recovers $\hat{L}_t$ as the first $n \times m$ block of $ \hat{A}_t^\dagger\mathbf{\hat{O}_t}^\dagger(\bar{A},C,d_1)\hat{\mathcal{H}}_t^-$. The following Theorem 4 of \citet{lale2020logarithmic} essentially translates the bound in Theorem \ref{theo:closedloopid} to individual bounds of system parameter estimates. It provides the high probability confidence sets required for deploying \OFU principle for the adaptive control. Note that Theorem 4 of \citet{lale2020logarithmic} does not estimate $L$ but in the current result we need to recover $L$ for controller construction. 

\begin{theorem}[Confidence Set Construction, \citep{lale2020logarithmic}] \label{ConfidenceSets}
Let $\mathcal{H}$ be the concatenation of two Hankel matrices obtained from $\mathbf{M}$. Let $\bar{A}, \bar{B}, \bar{C}, \bar{L}$ be the system parameters that \Sys provides for $\mathbf{M}$. At time step $t$, let $\hat{A}_t, \hat{B}_t, \hat{C}_t, \hat{L}_t$ denote the system parameters obtained by \Sys using the least squares estimate of the truncated ARX model, $\mathbf{\hat{M}_{t}}$. Suppose Assumptions \ref{Stable} and \ref{AssumContObs} hold, thus $\mathcal{H}$ is rank-$n$. After the warm-up period of $\Tw$, for the given choice of $H$, there exists a unitary matrix $\mathbf{T} \in \R^{n \times n}$ such that, with high probability, $\bar{\Theta}=(\bar{A}, \bar{B}, \bar{C}, \bar{L}) \in (\mathcal{C}_A \times \mathcal{C}_B \times \mathcal{C}_C \times \mathcal{C}_L) $ where
\begin{align}
    &\mathcal{C}_A(t) = \left \{A' \in \R^{n \times n} : \|\hat{A}_t - \mathbf{T}^\top A' \mathbf{T} \| \leq \beta_A(t) \right\}, \nonumber \\
    &\mathcal{C}_B(t) = \left \{B' \in \R^{n \times p} : \|\hat{B}_t - \mathbf{T}^\top B' \| \leq  \beta_B(t) \right\}, \nonumber \\
    &\mathcal{C}_C(t) = \left\{C' \in \R^{m \times n} : \|\hat{C}_t -  C'  \mathbf{T} \| \leq \beta_C(t) \right\}, \nonumber \\
    &\mathcal{C}_L(t) = \left \{L' \in \R^{p \times m} : \|\hat{L}_t - \mathbf{T}^\top L' \| \leq  \beta_L(t) \right\}, \nonumber
\end{align}
for
\begin{align}
&\beta_A(t) = c_1\left( \frac{\sqrt{nH}(\|\mathcal{H}\| + \sigma_n(\mathcal{H}))}{\sigma_n^2(\mathcal{H})} \right)\|\mathbf{\hat{M}_{t} } - \mathbf{M}\|, \quad \beta_B(t) = \beta_C  =  \sqrt{\frac{20nH }{\sigma_n(\mathcal{H})}}\|\mathbf{\hat{M}_{t} } - \mathbf{M}\|, \\
&\qquad \qquad \qquad \beta_L(t) =  \frac{c_2\|\mathcal{H} \|}{\sqrt{\sigma_n(\mathcal{H})}}\beta_A \!+\! c_3 \frac{\sqrt{nH}(\|\mathcal{H}\| + \sigma_n(\mathcal{H}))}{\sigma_n^{3/2}(\mathcal{H})}\|\mathbf{\hat{M}_{t} } \!-\! \mathbf{M}\|, \nonumber
\end{align}
for some problem dependent constants $c_1, c_2$ and $c_3$.
\end{theorem}

The proof is given in the Appendix \ref{SupSec:ConfSet}. It combines Lemma B.1 of \citet{oymak2018non} with careful perturbation analysis on the system parameter estimates provided by \Sys. 

\subsection{Adaptive Control}
\label{ControlDesign}
Using the confidence sets, \Alg implements \OFU principle. At time $t$, the algorithm chooses a system $\ttt_t = (\ta_t, \tb_t, \tc_t, \tl_t)$ from $\mathcal{C}_t \cap \mathcal{S}$ where $\mathcal{C}_t \coloneqq (\mathcal{C}_A(t) \times \mathcal{C}_B(t) \times \mathcal{C}_C(t) \times \mathcal{C}_L(t))$ such that 
\begin{equation} \label{optimistic}
    J(\ttt_t) \leq \inf_{\Theta' \in \mathcal{C}_t \cap \mathcal{S}} J(\Theta')+ 1/T.
\end{equation}
The algorithm designs the optimal feedback policy $(\tp_t, \tk_t, \tl_t)$ for the chosen system $\ttt_t$. It uses this optimistic controller to control the underlying system $\Theta$ for twice as long as the duration of the previous control policy. This technique known as ``doubling trick'' in reinforcement learning and online learning prevents frequent policy updates and balances the policy changes so that the overall regret of the algorithm is affected by a constant factor only.

\section{Regret Analysis of \Alg}
Now that the confidence set constructions and the adaptive control procedure of \Alg are explained, it only remains to analyze the regret of \Alg. Lemma 4.1 of \citet{lale2020regret} shows that the random exploration in the warm-up period acquires linear regret, \textit{i.e.} $\OO(\Tw)$. 

In order to analyze the regret obtained during the adaptive control period, we first need to show that system will be well-controlled during the adaptive control period. The following lemma achieves that. 
\begin{lemma} \label{Boundedness}
Suppose Assumptions \ref{Stable}-\ref{Stabilizable set} hold. After the warm-up period of $\Tw$, \Alg satisfies the following with high probability for all $T \geq t\geq \Tw$, 
\begin{enumerate}
    \item $\Theta \in (\mathcal{C}_A(t) \times \mathcal{C}_B(t) \times \mathcal{C}_C(t) \times \mathcal{C}_L(t)) $
    \item $\| \hat{x}_{t|t,\tth}\| \leq \tilde{\mathcal{X}}$
    \item $\| y_t \| \leq \tilde{\mathcal{Y}}$ 
\end{enumerate}
where $\tilde{\mathcal{X}}, \tilde{\mathcal{Y}} = \OO(\sqrt{\log(T)})$. Here, $\OO$ hides the problem dependent constants.  
\end{lemma}

The proof of the lemma with the precise expressions is given in Appendix \ref{SupSec:Boundedness}. This lemma is critical for the regret analysis due to the nature of the adaptive control problem in partially observable environments. The inaccuracies in the system parameter estimates affect both the optimal feedback gain synthesis and the estimation of the underlying state. If these inaccuracies are not tolerable in the adaptive control of the system, they will accumulate fast and cause explosion and unboundedness in the input and the output of the system. This would result in linear, and potentially super linear regret. The main technical challenge in the proof is to show that with $\Tw$ length warm-up period, the error between the optimistic controller's state estimation $\hat{x}_{t|t,\tth}$ and the true state estimation $\hat{x}_{t|t,\Theta}$ does not blow up. Lemma \ref{Boundedness} shows that while the system parameter estimates are refining, the input to the system and the system's output stays bounded during the adaptive control period. 

Given the verification of stability in the adaptive control period, we bound the regret of adaptive control. The regret analysis is based on the Bellman optimality equation for \LQG control problem provided in Lemma 4.3 of \citet{lale2020regret}. The following theorem gives the regret upper bound of the adaptive control period of \Alg.

\begin{theorem}[The regret of adaptive control] \label{adaptive control regret}
Suppose Assumptions \ref{Stable}-\ref{Stabilizable set} hold. After the warm-up period of $\Tw$, with high probability, for any time $T$ in adaptive control period, the regret of \Alg is bounded as follows:
\begin{equation}
    \reg(T) = \tilde{\OO}\left(\sqrt{T} \right).
\end{equation}
where $\tilde{\OO}(\cdot)$ hides the logarithmic factors and problem dependent constants.  
\end{theorem}

The proof is given in the Appendices \ref{SupSec:RegDecomp} and \ref{SupSec:RegAnaly}. Here we provide the main proof ideas. Since we know that the optimistic controller can attain smaller average expected cost than the optimal controller of the given system, we decompose the regret using the Bellman optimality equation for the optimistic system. For each time step $t$, $(\hat{x}_{t|t-1}, y_t)$ is treated as the given state of the system and the differences between the system evolutions of the true system and optimistic system are analyzed in the regret decomposition. The regret decomposition is given in Appendix \ref{SupSec:RegDecomp}. In Appendix \ref{SupSec:RegAnaly}, we bound each term individually. The main pieces are the facts that the confidence sets shrink with $\tilde{\OO}(1/\sqrt{t})$ (Theorem \ref{ConfidenceSets}), \Alg avoids frequent policy changes and the control inputs and system outputs are well-controlled (Lemma \ref{Boundedness}). Combining Theorem \ref{adaptive control regret} with $\OO(\Tw)$ regret from the warm-up period gives the overall regret upper bound of \Alg, stated in Theorem \ref{total regret main text}.

\section{Related Works}
The problem of sequential decision making under uncertainty is one the core studies in the field of control theory and reinforcement learning. Decision making in dynamical systems, when the environment is known and regulating costs are considered, results in a reduction to the study of optimal control. Optimal controls in the general setting of partially observable linear quadratic Gaussian systems, when highly crafted sensory observations of the system are available, and a fidelity approximation of the physics of dynamical systems is provided, has a long history of applications and successes.~\citep{aastrom2012introduction,bertsekas1995dynamic,hassibi1999indefinite}. 

When there is a high uncertainty in the modeling of the system, learning algorithms are required to learn the system behavior. In such situations, the learning agent estimates the system behaviour and adapt accordingly~\citep{ljung1999system, kailath2000linear}. For the class of fully observable systems, ~\citet{lai1982least,chen1987optimal} study this problem in asymptotic optimality sense, mainly developed on pure exploration approaches. Along with the regret analysis, the principle of pure exploration and betting on the best, or \ofu has been studied for fully observable environments~\citep{lai1985asymptotically,campi1998adaptive, bittanti2006adaptive}. Recent works, deploy the \ofu principle, and study tabular fully and partially observable Markov decision processes~\citep{jaksch2010near,azizzadenesheli2016reinforcement}. A seminal work by~\citet{abbasi2011regret} extends the \ofu principle and employ recent advances in the estimation theory~\citep{pena2009self,abbasi2011improved} and provide the first regret upper bound of $\tilde{\mathcal{O}}(\sqrt{T})$ for the fully observable case. In the setting of fully observable environments, an extensive advances and development have been proposed to provide generalized methods~\citep{faradonbeh2017optimism,abeille2017thompson,abeille2018improved,ouyang2017learning,dean2018regret}. Simultaneously, pure exploration methods along with uncertainty equivalence methods shed lights into the design of efficient algorithms ~\citep{abbasi2019model,mania2019certainty, faradonbeh2018input, cohen2019learning}. 

The system identification in partial observable linear systems in the presence of Gaussian noise, \LQG{}s, has recently sparked a flurry of research interests ~\citep{chen1992integrated,juang1993identification,phan1994system, lee2019robust,oymak2018non,sarkar2019finite,simchowitz2019learning,lee2019non,tsiamis2019finite,tsiamis2019sample, umenberger2019robust}. Most of the proposed methods in prior works utilize open-loop system identification methods (without a history dependent controller), using independent Gaussian excitation, which makes it easy to show the persistence of excitation and deal with the biases in the estimation using Markov parameters. However, in \citet{lee2019non}, the authors use the innovations form of the state-space model to deal with the biases in closed-loop system identification whereas in \citet{tsiamis2019finite}, it is shown that process and measurement noises are sufficient for persistence of excitation in the absence of a control input. Another line of novel approaches is proposed to extend the problem of estimation and prediction to online convex optimization where a set of strong theoretical guarantees on cumulative prediction errors are provided~\citep{hazan2017learning,arora2018towards,hazan2018spectral}. 

More recently, \citet{lale2020logarithmic} propose the first learning algorithm to estimate the model parameters using any arbitrary bounded sequence of samples, even with feedback controls where the future events are correlated with historical data. Along with the estimation, they provide statistically tight high probability confidence intervals over the model parameters where the true model parameters live in which is adopted in this work. They achieve the first logarithmic regret under the strong convexity assumption of the cost function similar to $\tilde{\mathcal{O}}(\sqrt{T})$ regret upper bound of \citet{simchowitz2020improper} which relies on the strong convexity and holds for semi-adversarial disturbances. In \citet{simchowitz2020improper}, the authors also consider the convex setting and attain $\tilde{\mathcal{O}}({T}^{2/3})$ regret upper bound even in the presence of adversarial disturbances

Another recent work by \citet{lale2020regret} provides a regret bound of $\tilde{\mathcal{O}}({T}^{2/3})$ for such problem with convex cost function. The current work, through deploying the novel estimation procedure improves the $\tilde{\mathcal{O}}({T}^{2/3})$ bound to $\tilde{\mathcal{O}}(\sqrt{T})$. These mentioned works and the current paper, are amongst the first to provide sublinear regret bounds for partially observable linear systems.

\section{Conclusion}
In this work, we study the problem of adaptive control in partially observable linear systems, also known as linear systems with imperfect observation. While the prior work relies on open-loop system identification, we adopt a novel method to estimate the system parameters even in the presence of feedback loop and correlation induced by feedback controllers. We deploy the principles of the Ho-Kalman method to estimate the model parameters and construct their corresponding confidence bound. We deploy the principle of optimism in the face of uncertainty and propose \alg, a reinforcement algorithm for \LQG{}s. \alg sequentially interacts with the environment for a few time steps, collect samples, and exploit the samples to estimate the model parameters up to their confidence sets. \alg computes the optimal controller associated with the most optimistic model in the set of plausible models, and then deploy this controller on the systems, but this time for a bit longer. \alg repeats this process. We show that following \alg results in a sublinear regret of $\tilde{\mathcal{O}}({\sqrt{T}})$ which is to the best of our knowledge the first ${\sqrt{T}}$ regret bound on \LQG with convex cost.

In future work, we also aim to utilize the estimation method in this work to study the safety in adaptive control. Along with safety, we plan to extend this work to the problem of constraint control. While the Gaussian assumption on the noise has been long considered for partially observable linear dynamical systems, this assumption introduces limitation and model mismatch. Due to the generality of estimation analysis proposed methods in this work, in the future work, we aim to extend the current results to the case of sub-Gaussian with unknown but bounded parameters. 

\section*{Acknowledgements}
S. Lale is supported in part by DARPA PAI. K. Azizzadenesheli is supported in part by Raytheon and Amazon Web Service. B. Hassibi is supported in part by the National Science Foundation under grants CNS-0932428, CCF-1018927, CCF-1423663 and CCF-1409204, by a grant from Qualcomm Inc., by NASA’s Jet Propulsion Laboratory through the President and Director’s Fund, and by King Abdullah University of Science and Technology. A. Anandkumar is supported in part by Bren endowed chair, DARPA PAIHR00111890035 and LwLL grants, Raytheon, Microsoft, Google, and Adobe faculty fellowships.

\bibliography{main}
\bibliographystyle{plainnat}

\newpage
\appendix

\begin{center}
{\huge Appendix}
\end{center}

In the following, we first provide the definitions of truncated noise evolution parameters for both warm-up period and adaptive control period in Appendix \ref{Gol-Gcl}. Appendix \ref{Gol-Gcl} also contains lower bounds on the smallest singular value for $\| \Phi_t \Phi_t^\top \| $ for warm-up period and adaptive control period which are used in showing persistence of excitation and thus proving Theorem \ref{theo:id2norm}. In Appendix \ref{SupSec:SysId}, we show how the self-normalized bound is obtained for $\mathbf{\hat{M}_t}$ and provide the proof of Theorem \ref{theo:closedloopid}. 

Appendix \ref{SupSec:ConfSet} gives the \Sys algorithm and describes the construction of confidence sets using the outputs of \Sys and provides the theoretical guarantees for them. In Appendix \ref{SupSec:Boundedness}, we give the proof of Lemma \ref{Boundedness} and show that with the given warm-up period, the inputs and the outputs of the system stay bounded with high probability. 

Appendix \ref{SupSec:RegDecomp} provides regret decomposition for \Alg and states the differences arise from the policy updates in adaptive control period compared to explore and commit algorithm proposed in \citet{lale2020regret}. In the Appendix \ref{SupSec:RegAnaly}, we provide the proof of regret upper bound for the adaptive control period of \Alg. Finally, in Appendix \ref{SupSec:NonSteady}, we give the overview of the case when the initial state for the system is not coming from the steady state distribution. 

Note that the warm-up period is chosen to be the following,
\begin{equation*}
    \Tw \geq \max \{T_A, T_B, T_c, T_o, T_u, T_{\mathbf{M}}, T_N, T_{\alpha}, T_{\beta}, T_{\gamma}, T_{\mathcal{G}}\}
\end{equation*}
where each term satisfies different condition in order to obtain $\tilde{\OO}(\sqrt{T})$ regret upper bound. The meanings of the terms are explained in detail throughout the Appendix.

\section{$H-$length Truncated Noise Evolution Parameters} \label{Gol-Gcl}

In this section, we provide definitions of truncated open-loop and closed-loop noise evolution parameters, $\Gol$ and $\Gcl$ respectively. They will play significant role in the confidence set for $\mathbf{M}$ in showing the persistence of excitation. They represent the effect of noises in the system on the outputs and the inputs. We will define $\Gol$ and $\Gcl$ for $2H$ time steps back in time and show that last $2H$ process and measurement noises provide sufficient persistent excitation for the covariates in the estimation problem. In the following, $\bar{\phi}_t = P \phi_t$ for a permutation matrix $P$ that gives 
\begin{equation*}
    \bar{\phi}_t = \left[ y_{t-1}^\top \enskip u_{t-1}^\top \ldots y_{t-H}^\top \enskip  u_{t-H}^\top \right]^\top \in \mathbb{R}^{(m+p)H}.
\end{equation*}

\subsection{Truncated Open-Loop Noise Evolution Parameter}\label{Golsubsection}

Recall the state-space form of the system,
\begin{align}
    x_{t+1} &= A x_t + B u_t + w_t \nonumber \\
    y_t &= C x_t + z_t. \label{system_apx}
\end{align}
During the warm-up period, $t \leq \Tw$, the input to the system is $u_t \sim \mathcal{N}(0,\sigma_u^2 I)$. Let $f_t = [y_t^\top u_t^\top]^\top$. From the evolution of the system with given input we have the following:
\begin{equation*}
    f_t = \mathbf{G^o} \begin{bmatrix}
    w_{t-1}^\top & z_t^\top & u_t^\top & \ldots & w_{t-H}^\top & z_{t-H+1}^\top & u_{t-H+1}^\top
\end{bmatrix}^\top + \mathbf{r_t^o}
\end{equation*}
where 
\begin{align}
\!\!\!\mathbf{G^o}\!\! := \!\!
\begin{bmatrix}
0_{m\!\times\! n}~I_{m\!\times\! m}~0_{m\!\times\! p}  & C~0_{m\!\times\! m}~CB & CA~0_{m\!\times\! m}~CAB &\ldots & \quad CA^{H-2}~0_{m\!\times\! m}~CA^{H-2}B \\
0_{p\!\times\! n}~0_{p\!\times\! m}~I_{p\!\times\! p}  &0_{p\!\times\! n}~0_{p\!\times\! m}~0_{p\!\times\! p}  & 0_{p\!\times\! n}~0_{p\!\times\! m}~0_{p\!\times\! p} &\ldots& 0_{p\!\times\! n}~0_{p\!\times\! m}~0_{p\!\times\! p} 
\end{bmatrix}
\end{align}
and $\mathbf{r_t^o}$ is the residual vector that represents the effect of $[w_{i-1} \enskip z_i \enskip u_i ]$ for $0 \leq i<t-H$, which are independent. Notice that $\mathbf{G^o}$ is full row rank even for $H=1$, due to first $(m+p)\times (m+n+p)$ block. Using this, we can represent $\bar{\phi}_t$ as follows
\begin{align}
    \bar{\phi}_t &= \underbrace{\begin{bmatrix}
    f_{t-1} \\
    \vdots \\
    f_{t-H}
\end{bmatrix}}_{\mathbb{R}^{(m+p)H}} + 
\begin{bmatrix}
    \mathbf{r_{t-1}^o} \\
    \vdots \\
    \mathbf{r_{t-H}^o}
\end{bmatrix}
= \Gol
\underbrace{\begin{bmatrix}
    w_{t-2} \\
    z_{t-1} \\
    u_{t-1}\\
    \vdots \\
    w_{t-2H-1}\\
    z_{t-2H} \\
    u_{t-2H}
\end{bmatrix}}_{\mathbb{R}^{2(n+m+p)H}} + 
\begin{bmatrix}
    \mathbf{r_{t-1}^o} \\
    \vdots \\
    \mathbf{r_{t-H}^o}
\end{bmatrix} \quad \text{ where } \nonumber \\
\Gol &\coloneqq \!\!
\begin{bmatrix}
    [\qquad \qquad \mathbf{G^o} \qquad \qquad] \quad 0_{(m+p)\times (m+n+p)} \enskip 0_{(m+p)\times (m+n+p)} \enskip 0_{(m+p)\times (m+n+p)} \enskip \ldots \\
    0_{(m+p)\times (m+n+p)} \enskip [\qquad \qquad \mathbf{G^o} \qquad \qquad] \qquad  0_{(m+p)\times (m+n+p)}  \enskip 0_{(m+p)\times (m+n+p)} \enskip \ldots \\
    \ddots \\
     0_{(m+p)\times (m+n+p)}  \enskip 0_{(m+p)\times (m+n+p)} \enskip \ldots \quad [\qquad \qquad \mathbf{G^o} \qquad \qquad] \enskip 0_{(m+p)\times (m+n+p)} \\
    0_{(m+p)\times (m+n+p)} \enskip 0_{(m+p)\times (m+n+p)} \enskip 0_{(m+p)\times (m+n+p)} \enskip \ldots \qquad [\qquad \qquad \mathbf{G^o} \qquad \qquad]
\end{bmatrix}.
\label{Gol}
\end{align}
Define 
\begin{equation*}
    T_o = \frac{32 \Upsilon_w^4 \log^2\left(\frac{2H(m+p)}{\delta}\right)}{\sigma_{\min}^4(\Gol) \min \{ \sigma_w^4, \sigma_z^4, \sigma_u^4 \}}.
\end{equation*}
We now prove Lemma \ref{openloop_persistence}, which shows that the inputs are persistently exciting uniformly during the warm-up period for $t \geq T_o$. 
\vspace{0.4cm}

\begin{lemma}[Precise Statement of Lemma \ref{openloop_persistence}] \label{openloop_persistence_appendix}
If the warm-up duration $\Tw \geq T_o$, then for $T_o \leq t\leq \Tw$, with probability at least $1-\delta$ we have
\begin{equation}
\sigma_{\min}\left(\sum_{i=1}^{t} \phi_i \phi_i^\top \right) \geq t \frac{\sigma_{o}^2 \min \{ \sigma_w^2, \sigma_z^2, \sigma_u^2 \}}{2}.  
\end{equation}
\end{lemma}

\begin{proof}
Let $\Bar{\mathbf{0}} = 0_{(m+p)\times (m+n+p)}$. Since each block row is full row-rank, we get the following decomposition using QR decomposition for each block row: 
\begin{align*}
    \Gol = \underbrace{\begin{bmatrix}
   Q^{o} & 0_{m+p} & 0_{m+p} & 0_{m+p} & \ldots \\
    0_{m+p} & Q^{o} &  0_{m+p} & 0_{m+p} & \ldots \\
    & & \ddots & & \\
     0_{m+p}  & 0_{m+p} & \ldots & Q^{o} & 0_{m+p} \\
    0_{m+p} & 0_{m+p} & 0_{m+p} & \ldots & Q^{o}
\end{bmatrix} }_{\mathbb{R}^{(m+p)H \times (m+p)H }}
\underbrace{\begin{bmatrix}
   R^{o} \!&\! \Bar{\mathbf{0}} \!&\! \Bar{\mathbf{0}} \!&\! \Bar{\mathbf{0}} \!&\! \ldots \\
    \Bar{\mathbf{0}} \!&\! R^{o} \!&\!  \Bar{\mathbf{0}} \!&\! \Bar{\mathbf{0}} \!&\! \ldots \\
    \!&\! \!&\! \ddots \!&\! \!&\! \\
     \Bar{\mathbf{0}}  \!&\! \Bar{\mathbf{0}} \!&\! \ldots \!&\!R^{o} \!&\! \Bar{\mathbf{0}} \\
    \Bar{\mathbf{0}} \!&\! \Bar{\mathbf{0}} \!&\! \Bar{\mathbf{0}} \!&\! \ldots \!&\! R^{o}
\end{bmatrix}}_{\mathbb{R}^{(m+p)H \times 2(m+n+p)H }}
\end{align*}
where $R^{o} = \begin{bmatrix}
       \x & \x & \x & \x & \x & \x &\ldots \\ 0 & \x & \x & \x & \x & \x & \ldots \\ & \ddots & \\ 0 & 0 & 0 & \x & \x & \x & \ldots \end{bmatrix} \in \mathbb{R}^{(m+p) \times H(m+n+p) }$ 
where the elements in the diagonal are positive numbers. Notice that the first matrix with $Q^{0}$ is full rank. Also, all the rows of second matrix are in row echelon form and second matrix is full row-rank. Thus, we can deduce that $\Gol$ is full row-rank. Since $\Gol$ is full row rank, we have that 
\begin{equation*}
    \mathbb{E}[\bar{\phi}_t \bar{\phi}_t^\top] \succeq \Gol \Sigma_{w,z,u} \mathcal{G}^{ol \top}
\end{equation*}
where $\Sigma_{w,z,u} \in \R^{2(n+m+p)H \times 2(n+m+p)H} = \text{diag}(\sigma_w^2, \sigma_z^2, \sigma_u^2, \ldots,\sigma_w^2, \sigma_z^2, \sigma_u^2)$. This gives us 
\[ 
\sigma_{\min}(\mathbb{E}[\bar{\phi}_t \bar{\phi}_t^\top]) \geq \sigma_{\min}^2(\Gol) \min \{ \sigma_w^2, \sigma_z^2, \sigma_u^2 \}
\]
for $t \leq \Tw$. As given in (\ref{exploration norms first})-(\ref{exploration norms last}), we have that $\|\phi_t\| \leq \Upsilon_w \sqrt{H}$ with probability at least $1-\delta/2$. Given this holds, one can use Theorem \ref{azuma}, to obtain the following which holds with probability $1-\delta/2$:
\begin{align*}
    \lambda_{\max}\left (\sum_{i=1}^{t} \phi_i \phi_i^\top - \mathbb{E}[\phi_i \phi_i^\top]  \right) \leq 2\sqrt{2t} \Upsilon_w^2 H \sqrt{\log\left(\frac{2H(m+p)}{\delta}\right)}. 
\end{align*}
Using Weyl's inequality, during the warm-up period with probability $1-\delta$, we have 
\begin{align*}
    \sigma_{\min}\left(\sum_{i=1}^{t} \phi_i \phi_i^\top \right) \geq t \sigma_{o}^2\min \{ \sigma_w^2, \sigma_z^2, \sigma_u^2 \} - 2\sqrt{2t} \Upsilon_w^2 H \sqrt{\log\left(\frac{2H(m+p)}{\delta}\right)}.
\end{align*}
For all $t\geq T_{o} \coloneqq \frac{32 \Upsilon_w^4 H^2 \log\left(\frac{2H(m+p)}{\delta}\right)}{\sigma_{o}^4 \min \{ \sigma_w^4, \sigma_z^4, \sigma_u^4 \}}$, we have the stated lower bound.
\end{proof}

\subsection{Truncated Closed-Loop Noise Evolution Parameter}
\label{Gclsubsection}
After the warm-up period, for $t \geq \Tw$, the input to the system is $u_t = -\tk_t \hat{x}_{t|t,\ttt}$. Recall the following relation for state estimation updates using the optimistic parameters: 
\begin{align}
    \hat{x}_{t|t-1,\ttt} &= \ta_{t-1} \hat{x}_{t-1|t-1,\ttt} - \tb_{t-1} \tk_{t-1} \hat{x}_{t-1|t-1,\ttt} \nonumber \\
    \hat{x}_{t|t,\ttt} &= \hat{x}_{t|t-1,\ttt} + \tl_t(y_t - \tc_t \hat{x}_{t|t-1,\ttt}) \nonumber \\
    &= (\ta_{t-1} - \tb_{t-1} \tk_{t-1}) \hat{x}_{t-1|t-1,\ttt} + \tl_t(Cx_t + z_t - \tc_t(\ta_{t-1} - \tb_{t-1} \tk_{t-1})\hat{x}_{t-1|t-1,\ttt} ) \nonumber \\
    &= (I - \tl_t \tc_t)(\ta_{t-1} - \tb_{t-1} \tk_{t-1}) \hat{x}_{t-1|t-1,\ttt} + \tl_t(C(A x_{t-1} - B \tk_{t-1} \hat{x}_{t-1|t-1,\ttt} + w_{t-1}) + z_t ). \label{xtt estimation}
\end{align}
Again, let $f_t = [y_t^\top u_t^\top]^\top$. Using (\ref{system_apx}) and (\ref{xtt estimation}), the following can be written for $f_t$
\begin{equation*}
    \begin{bmatrix}
    x_{t} \\
    \hat{x}_{t|t,\ttt}
\end{bmatrix} = 
\underbrace{\begin{bmatrix}
    A & -B \tk_{t-1} \\
    \tl_t CA & (I\!-\!\tl_t \tc_t)(\ta_{t-1} \!-\! \tb_{t-1} \tk_{t-1}) \!-\! \tl_t C B \tk_{t-1}
\end{bmatrix}}_{\mathbf{\tilde{G}_2^{(t)}}}
\begin{bmatrix}
    x_{t-1} \\
    \hat{x}_{t-1|t-1,\ttt} 
\end{bmatrix} + 
\underbrace{ \begin{bmatrix}
    I & 0 \\
    \tl_t C & \tl_t
\end{bmatrix}}_{\mathbf{\tilde{G}_3^{(t)}}}
\begin{bmatrix}
    w_{t-1} \\
    z_t 
\end{bmatrix}
\end{equation*}
\begin{equation*}
    f_t = 
\begin{bmatrix}
    CA & -CB \tk_{t-1} \\
    -\tk_t \tl_t CA & -\tk_t(I\!-\!\tl_t \tc_t)(\ta_{t-1} \!-\! \tb_{t-1} \tk_{t-1}) \!+\! \tk_t \tl_t C B \tk_{t-1}
\end{bmatrix}
\begin{bmatrix}
    x_{t-1} \\
    \hat{x}_{t-1|t-1,\ttt} 
\end{bmatrix} + 
\begin{bmatrix}
    Cw_{t-1} + z_t \\
    -\tk_t \tl_t (z_t \!+\! Cw_{t-1})
\end{bmatrix}
\end{equation*}
\begin{equation*}
   f_t = 
\underbrace{ \begin{bmatrix}
    C & 0 \\
    0 & -\tk_t
\end{bmatrix}}_{\mathbf{\tilde{\Gamma}_t}}
\mathbf{\tilde{G}_2^{(t)}}
\begin{bmatrix}
    x_{t-1} \\
    \hat{x}_{t-1|t-1,\ttt} 
\end{bmatrix} + 
\underbrace{ \begin{bmatrix}
    C & 0 \\
    0 & -\tk_t
\end{bmatrix}}_{\mathbf{\tilde{\Gamma}_t}}
\underbrace{ \begin{bmatrix}
    I & 0 \\
    \tl_t C & \tl_t
\end{bmatrix}}_{\mathbf{G_3^{(t)}}}
\begin{bmatrix}
    w_{t-1} \\
    z_t 
\end{bmatrix} + 
\begin{bmatrix}
    z_t \\
    0 
\end{bmatrix}.
\end{equation*}
Rolling back in time for $H$ time steps we get the following,
\begin{equation*}
    f_t = \mathbf{\tilde{\Gamma}_t} \left( \sum_{i=t-H+1}^{t} \left(\prod_{j=i}^{t} \mathbf{\tilde{G}_2^{(j)}}\right) \mathbf{\tilde{G}_3^{(i-1)}}  \begin{bmatrix}
    w_{i-2} \\
    z_{i-1} 
\end{bmatrix} \right)+ \underbrace{ \begin{bmatrix}
    C & I \\
    -\tk_t \tl_t C & -\tk_t \tl_t
\end{bmatrix}}_{\mathbf{\tilde{G}_1^{(t)}}}
\begin{bmatrix}
    w_{t-1} \\
    z_t 
\end{bmatrix} + \mathbf{r_t^c}
\end{equation*}
where $\mathbf{r_t^c}$ is the residual vector that represents the effect of $[w_{i-1} \enskip z_i]$ for $0 \leq i<t-H$, which are independent. Using this, we can represent $\bar{\phi}_t$ as follows
\begin{equation*}
    \bar{\phi}_t = \underbrace{\begin{bmatrix}
    f_{t-1} \\
    \vdots \\
    f_{t-H}
\end{bmatrix}}_{\mathbb{R}^{(m+p)H}} + \begin{bmatrix}
    \mathbf{r_{t-1}^c} \\
    \vdots \\
    \mathbf{r_{t-H}^c}
\end{bmatrix} = \Gcl_t
\underbrace{\begin{bmatrix}
    w_{t-2} \\
    z_{t-1} \\
    \vdots \\
    w_{t-2H-1}\\
    z_{t-2H}
\end{bmatrix}}_{\mathbb{R}^{2(n+m)H}} + \begin{bmatrix}
    \mathbf{r_{t-1}^c} \\
    \vdots \\
    \mathbf{r_{t-H}^c}
\end{bmatrix}
\end{equation*}
where
\begin{equation}
    \Gcl_t = \begin{bmatrix}[\qquad \enskip \mathbf{\bar{G}_{t-1}} \enskip \qquad] \qquad 0_{(m+p)\times (m+n)} \enskip 0_{(m+p)\times (m+n)} \enskip 0_{(m+p)\times (m+n)} \enskip \ldots \\
    0_{(m+p)\times (m+n)} \enskip [\qquad \enskip \mathbf{\bar{G}_{t-2}} \enskip \qquad] \qquad  0_{(m+p)\times (m+n)}  \enskip 0_{(m+p)\times (m+n)} \enskip \ldots \\
    \ddots \\
     0_{(m+p)\times (m+n)}  \enskip 0_{(m+p)\times (m+n)} \enskip \ldots \quad [\qquad \enskip \mathbf{\bar{G}_{t-H+1}} \enskip \qquad] \enskip 0_{(m+p)\times (m+n)} \\
    0_{(m+p)\times (m+n)} \enskip 0_{(m+p)\times (m+n)} \enskip 0_{(m+p)\times (m+n)} \enskip \ldots \qquad [\qquad \enskip \mathbf{\bar{G}_{t-H}} \enskip \qquad]
    \end{bmatrix}
\end{equation}
for 
\begin{equation*}
    \mathbf{\bar{G}_{t}} \!=\! \begin{bmatrix}
    \mathbf{\tilde{G}_1^{(t)}}, \!&\! \mathbf{\tilde{\Gamma}_t} \mathbf{\tilde{G}_2^{(t)}} \mathbf{\tilde{G}_3^{(t-1)}}, \!&\! \mathbf{\tilde{\Gamma}_t} \mathbf{\tilde{G}_2^{(t)}} \mathbf{\tilde{G}_2^{(t-1)}} \mathbf{\tilde{G}_3^{(t-2)}},\ldots, \!&\! \mathbf{\tilde{\Gamma}_t} \mathbf{\tilde{G}_2^{(t)}} \mathbf{\tilde{G}_2^{(t-1)}}\!\!\!\!\!\!\!\! \ldots \mathbf{\tilde{G}_2^{(t-H+1)}} \mathbf{\tilde{G}_3^{(t-H)}}
\end{bmatrix} \in \mathbb{R}^{(m+p) \times H(n+m)}
\end{equation*}
By knowing the underlying system, the agent can deploy the optimal control policy. $\Gcl$ represents the translation of the process and measurement noises into $\bar{\phi}_t$ while using the optimal policy: 
\begin{align}
\Gcl &= \begin{bmatrix}
    [\qquad \qquad  \mathbf{\bar{G}} \qquad \qquad] \qquad 0_{(m+p)\times (m+n)} \qquad 0_{(m+p)\times (m+n)} \qquad 0_{(m+p)\times (m+n)} \qquad \ldots \\
    0_{(m+p)\times (m+n)} \qquad [\qquad \qquad \mathbf{\bar{G}} \qquad \qquad] \qquad  0_{(m+p)\times (m+n)}  \qquad 0_{(m+p)\times (m+n)} \qquad \ldots \\
    \ddots \\
     0_{(m+p)\times (m+n)}  \qquad 0_{(m+p)\times (m+n)} \qquad \ldots \qquad [\qquad \qquad \mathbf{\bar{G}} \qquad \qquad] \qquad 0_{(m+p)\times (m+n)} \\
    0_{(m+p)\times (m+n)} \qquad 0_{(m+p)\times (m+n)} \qquad 0_{(m+p)\times (m+n)} \qquad \ldots \qquad [\qquad \qquad \mathbf{\bar{G}} \qquad \qquad]
\end{bmatrix}
\label{Gcl}
\end{align}
where
\begin{align*}
     \mathbf{\bar{G}} &= \begin{bmatrix}
    \mathbf{G_1}, & \mathbf{\Gamma} \mathbf{G_2} \mathbf{G_3}, & \mathbf{\Gamma} \mathbf{G_2}^2 \mathbf{G_3}, & \ldots, & \mathbf{\Gamma} \mathbf{G_2}^{H-1}\mathbf{G_3} 
\end{bmatrix}\in \mathbb{R}^{(m+p) \times H(n+m)}
\end{align*}
for 
\begin{equation*}
  \mathbf{G_1} \!=\! \begin{bmatrix}
    C \!&\! I \\
    -K L C \!&\! -K L
\end{bmatrix}, \mathbf{\Gamma} \!=\! \begin{bmatrix}
    C \!&\! 0 \\
    0 \!&\! -K
\end{bmatrix}, \mathbf{G_2} \!=\! \begin{bmatrix}
    A \!&\! -B K \\
    L CA \!&\! (I\!-\!L C)(A \!-\! B K) \!-\! L C B K
\end{bmatrix}, \mathbf{G_3} \!=\! \begin{bmatrix}
    I \!&\! 0 \\
    L C \!&\! L
\end{bmatrix}. 
\end{equation*}

Note that length of H is chosen such that $\mathbf{\bar{G}}$ is full row rank. Similar to the case with truncated open-loop noise evolution parameter, having full row rank block rows provides a full row rank $\Gcl$ via the same QR decomposition argument. Thus, the assumption on the lower bound of the smallest singular value of the $H-$length truncated closed-loop noise evolution parameter, $\sigma_{\min}(\Gcl) > \sigma_c > 0$, is valid. Due to boundedness of the set $\mathcal{S}$ that \Alg is searching on, let $\|\tilde{\Gcl} \|_F \leq G $ for all model in $\mathcal{S}$. Define $G_r = G + \frac{\sigma_{c}\sqrt{H(m+p)}}{2}$ and \[
T_c = \frac{2048\Upsilon_c^4H^2 \left( \log\left( \frac{H(m+p)}{\delta} \right)  +  H^2(m+p)(m + n)\log\left(G_r  +  \frac{ 32H\Upsilon_c\sqrt{2} \eta_T  +  32H \eta_T^2  +  16\max \{ \sigma_w^2, \sigma_z^2 \}}{\sigma_{c}^2\min \{ \sigma_w^2, \sigma_z^2 \}} \right) \right)}{\sigma_{c}^4 \min \{ \sigma_w^4, \sigma_z^4 \}}.
\] We now prove Lemma \ref{closedloop_persistence}, which shows that the inputs are persistently exciting uniformly during the adaptive control period for $t \geq T_c$. 
\vspace{0.3cm}
\begin{lemma}[Precise Statement of Lemma \ref{closedloop_persistence}] \label{closedloop_persistence_appendix}
After $T_c$ time steps in adaptive control period, with probability $1-3\delta$, we have the following for all $t\geq T_c$, 
\begin{equation}
\sigma_{\min}\left(\sum_{i=1}^{t} \phi_i \phi_i^\top \right) \geq t \frac{\sigma_{c}^2 \min \{ \sigma_w^2, \sigma_z^2\}}{16}. 
\end{equation}
\end{lemma}
\begin{proof}

Define $\tilde{\Gcl}$, which is the translation parameter for the process and measurement noises into $\bar{\phi}_t$ for the system that is governed by the \textbf{optimistically chosen parameter by} \Alg \textbf{while using the optimal optimistic controller}. Recall that we are searching for the optimistic system model which attains the optimal \LQG cost over the set of $\mathcal{C}_t \cap \mathcal{S}$ and whose closed-loop noise evolution parameter satisfies the lower bound on the smallest singular value of the $H-$length truncated closed-loop noise evolution parameter, $\sigma_{c}$. Therefore, \Alg has the guarantee that $\sigma_{\min}(\tilde{\Gcl}) \geq \sigma_{c}$. Let 
\begin{equation*}
    T_{\mathcal{G}} = T_B \left(\frac{2H + 2H\Gamma\zeta + 2H(H-1)\Gamma\zeta}{\sigma_c}\right)^2.
\end{equation*}
Picking $\Tw \geq T_{\mathcal{G}}$, guarantees that in adaptive control period for all $t\geq \Tw$, $\|\Gcl_t - \tilde{\Gcl} \| \leq \frac{\sigma_{c}}{2}$. Using Weyl's inequality on singular values, we have that $\sigma_{\min}(\Gcl_t) \geq  \frac{\sigma_{c}}{2}$. Hence, for all $t\geq \Tw$, we have that 
\begin{equation*}
    \mathbb{E}[\bar{\phi}_t \bar{\phi}_t^\top] \succeq \Gcl_t \Sigma_{w,z} \mathcal{G}_t^{cl \top}
\end{equation*}
where $\Sigma_{w,z} \in \R^{2(n+m)H \times 2(n+m)H} = \text{diag}(\sigma_w^2, \sigma_z^2, \ldots,\sigma_w^2, \sigma_z^2)$. This gives us $\sigma_{\min}(\mathbb{E}[\bar{\phi}_t \bar{\phi}_t^\top]) \geq \frac{\sigma_{c}^2}{4} \min \{ \sigma_w^2, \sigma_z^2 \}$ for $t \geq \Tw$. 
As given in (\ref{adaptive control norms first})-(\ref{adaptive control norms last}), we have that $\|\phi_t\| \leq \Upsilon_c \sqrt{H}$ with probability at least $1-2\delta$. Given this holds, \textbf{for a given optimistic model}, one can use Theorem \ref{azuma} as in the truncated open-loop noise evolution parameter, to obtain the following which holds with probability $1-\delta$:
\begin{align}
    \lambda_{\max}\left (\sum_{i=1}^{t} \phi_i \phi_i^\top - \mathbb{E}[\phi_i \phi_i^\top]  \right) \leq 2\sqrt{2t} \Upsilon_c^2 H \sqrt{\log\left(\frac{H(m+p)}{\delta}\right)}.
\end{align}
Notice that this bound holds only for a single model. However, we need to show that for any random model within the confidence set, the lower bound holds. Thus, we need a standard covering argument. Using the perturbation result that holds for all $t\geq \Tw$, we have $\|\Gcl_t \|_F \leq G_r$. We have the following upper bound on the covering number:
\begin{equation*}
    \mathcal{N}(B(G_r), \|\cdot \|_F, \epsilon) \leq \left(G_r + \frac{2}{\epsilon} \right)^{(m+p)(n+m)H^2}.
\end{equation*}
Thus, the following holds for all the centers of $\epsilon$-balls in $\|\Gcl_t\|_F$, for all $t\geq \Tw$, with probability $1-\delta$:
\begin{align}
    \lambda_{\max}\left (\sum_{i=1}^{t} \phi_i \phi_i^\top - \mathbb{E}[\phi_i \phi_i^\top]  \right) \leq 2\sqrt{2t} \Upsilon_c^2 H \sqrt{\log\left(\frac{H(m+p)}{\delta}\right) + H^2(m+p)(m+n)\log\left(G_r + \frac{2}{\epsilon}\right)}.
\end{align}

Let $\eta_T = \sigma_w \sqrt{2n\log\left(\frac{2nT}{\delta}\right)} + \sigma_z\sqrt{2m\log\left(\frac{2mT}{\delta}\right)}$. Considering all the systems in the $\epsilon$-balls, during the adaptive control period with probability $1-3\delta$, we have  
\begin{align*}
    \sigma_{\min}\left(\sum_{i=1}^{t} \phi_i \phi_i^\top \right) &\geq t \left( \frac{\sigma_{c}^2}{4} \min \{ \sigma_w^2, \sigma_z^2 \} - 2\epsilon \left( H\Upsilon_c\sqrt{2} \eta_T + H \eta_T^2 + \max \{ \sigma_w^2/2, \sigma_z^2/2 \} \right) \right) \\
    &\quad\qquad\qquad- 2\sqrt{2t} \Upsilon_c^2 H \sqrt{\log\left(\frac{H(m+p)}{\delta}\right) + H^2(m+p)(m+n)\log\left(G_r + \frac{2}{\epsilon}\right)}.
\end{align*}
Let $\epsilon = \frac{\sigma_{c}^2\min \{ \sigma_w^2, \sigma_z^2 \}}{16 \left( H\Upsilon_c\sqrt{2} \eta_T + H \eta_T^2 + \max \{ \sigma_w^2/2, \sigma_z^2/2 \} \right)}$. This gives the following bound 
\begin{align*}
    &\sigma_{\min}\left(\sum_{i=1}^{t} \phi_i \phi_i^\top \right) \geq t \left( \frac{\sigma_{c}^2}{8} \min \{ \sigma_w^2, \sigma_z^2 \} \right) \\
    &- 2\sqrt{2t} \Upsilon_c^2 H\! \sqrt{\log\left(\!\frac{H(m+p)}{\delta}\!\right) \!+\! H^2(m\!+\!p)(m\!+\!n)\log\left(G_r \!+\! \frac{ 32H\Upsilon_c\sqrt{2} \eta_T \!+\! 32H \eta_T^2 \!+\! 16\max \{ \sigma_w^2, \sigma_z^2 \}}{\sigma_{c}^2\min \{ \sigma_w^2, \sigma_z^2 \}} \right)}.
\end{align*}

For all $t\geq T_c$, we have the stated lower bound.
\end{proof}
\newpage

\section{System Identification}
\label{SupSec:SysId}

The results in this section is adapted from \citet{lale2020logarithmic}. Recall that for a single input-output trajectory $\{y_t, u_t \}^{T}_{t=1}$, using the ARX model, we can write the following for the given system,
\begin{equation} 
    Y_t = \Phi_t \mathbf{M}^{\top} + \underbrace{ E_t + N_t}_{ \text{Noise}} \qquad \text{where}
\end{equation}
\vspace{-0.8cm}
\begin{align*}
    \mathbf{M} & = \left[CF,\enskip C\Bar{A}F,\enskip \ldots, \enskip C \Bar{A}^{H-1}F, \enskip CB, \enskip C\Bar{A}B, \enskip \ldots, \enskip C\Bar{A}^{H-1} B \right] \in \R^{m \times (m+p)H} \\
    Y_t &= \left[ y_H,~y_{H+1},~ \ldots,~y_t \right]^\top \in \mathbb{R}^{(t-H) \times m} \\
    \Phi_t &= \left[ \phi_H,~\phi_{H+1},~\ldots,~\phi_t \right]^\top \in \mathbb{R}^{(t-H) \times (m+p)H} \\
    E_t &= \left[ e_H,~e_{H+1},~\ldots,~e_t \right]^\top \in \mathbb{R}^{(t-H) \times m} \\
    N_t\!&=\! \left[C \Bar{A}^H\!x_{0},~C \Bar{A}^H\!x_{1},\ldots,C \Bar{A}^H \!x_{t-H} \right]^\top \!\!\in\! \mathbb{R}^{(t-H) \times m}  
\end{align*}
$\mathbf{\hat{M}_{t} }$ is the solution to $\min_X \|Y_{t} - \Phi_{t} X^{\top}\|^2_F + \lambda \|X\|^2_F $. Hence, we get $\mathbf{\hat{M}_{t} }^\top = (\Phi_t^\top \Phi_t + \lambda I)^{-1} \Phi_t^\top Y_t $.  \\

\noindent\textbf{Proof of Theorem \ref{theo:closedloopid}}
\begin{align*}
    \mathbf{\hat{M}_{t}}  & = \big[(\Phi_t^\top \Phi_t + \lambda I)^{-1} \Phi_t^\top(\Phi_t \mathbf{M}^{\top} + E_t + N_t)\big]^{\top} \\
    &= \big[(\Phi_t^\top \Phi_t + \lambda I)^{-1} \Phi_t^\top \left(E_t + N_t \right) + (\Phi_t^\top \Phi_t + \lambda I)^{-1} \Phi_t^\top \Phi_t \mathbf{M}^{\top} \\
    &\qquad + \lambda (\Phi_t^\top \Phi_t + \lambda I)^{-1}\mathbf{M}^{\top} - \lambda (\Phi_t^\top \Phi_t + \lambda I)^{-1}\mathbf{M}^{\top} \big]^{\top} \\
    &= \big[(\Phi_t^\top \Phi_t + \lambda I)^{-1} \Phi_t^\top E_t + (\Phi_t^\top \Phi_t + \lambda I)^{-1} \Phi_t^\top N_t + \mathbf{M}^{\top} - \lambda (\Phi_t^\top \Phi_t + \lambda I)^{-1}\mathbf{M}^{\top} \big]^{\top}
\end{align*}
Using $\mathbf{\hat{M}_{t}}$, we get 
\begin{align}
    &|\Tr(X(\mathbf{\hat{M}_{t}}-\mathbf{M})^{\top})| \\
    &= |\Tr(X (\Phi_t^\top \Phi_t + \lambda I)^{-1} \Phi_t^\top E_t) + \Tr(X (\Phi_t^\top \Phi_t + \lambda I)^{-1} \Phi_t^\top N_t)   - \lambda \Tr(X (\Phi_t^\top \Phi_t + \lambda I)^{-1}\mathbf{M}^{\top})| \nonumber \\
    &\leq |\Tr(X (\Phi_t^\top \Phi_t + \lambda I)^{-1} \Phi_t^\top E_t)| + |\Tr(X (\Phi_t^\top \Phi_t + \lambda I)^{-1} \Phi_t^\top N_t)| + \lambda |\Tr(X (\Phi_t^\top \Phi_t + \lambda I)^{-1}\mathbf{M}^{\top})| \nonumber \\
    &\leq \sqrt{\Tr(X(\Phi_t^\top \Phi_t + \lambda I)^{-1} X^{\top})\Tr(E_t^{\top}\Phi_t(\Phi_t^{\top}\Phi_t + \lambda I)^{-1}\Phi_t^{\top}E_t )}  \label{CS_trace} \\
    &\quad + \sqrt{\Tr(X(\Phi_t^\top \Phi_t + \lambda I)^{-1} X^{\top})\Tr(N_t^{\top}\Phi_t(\Phi_t^{\top}\Phi_t + \lambda I)^{-1}\Phi_t^{\top}N_t )} \nonumber \\
    &\quad + \lambda \sqrt{\Tr(X(\Phi_t^{\top}\Phi_t + \lambda I)^{-1}X^{\top})\Tr(\mathbf{M}(\Phi_t^{\top}\Phi_t + \lambda I)^{-1}\mathbf{M}^{\top})} \nonumber \\
    &= \sqrt{\Tr(X(\Phi_t^{\top}\Phi_t + \lambda I)^{-1}X^{\top})} \enskip \times \nonumber \\
    &\bigg[\!\sqrt{\Tr(E_t^{\top}\Phi_t(\Phi_t^{\top}\Phi_t \!+\! \lambda I)^{-1}\Phi_t^{\top}E_t )} \!+\!\! \sqrt{\Tr(N_t^{\top}\Phi_t(\Phi_t^{\top}\Phi_t \!+\! \lambda I)^{-1}\Phi_t^{\top}N_t )} \!+\! \lambda \sqrt{\Tr(\mathbf{M}(\Phi_t^{\top}\Phi_t \!+\! \lambda I)^{-1}\mathbf{M}^{\top})} \bigg] \nonumber
\end{align}
where (\ref{CS_trace}) follows from $|\Tr(ABC^{\top})| \leq \sqrt{\Tr(ABA^{\top})\Tr(CBC^{\top}) }$ for positive definite B due to Cauchy Schwarz (weighted inner-product).  For $X = (\mathbf{\hat{M}_{t}}-\mathbf{M})(\Phi_t^{\top}\Phi_t + \lambda I)$, we get
\begin{align}
  \sqrt{\Tr((\mathbf{\hat{M}_{t}}-\mathbf{M})V_t(\mathbf{\hat{M}_{t}}-\mathbf{M})^{\top})} &\leq  \sqrt{\Tr(E_t^{\top}\Phi_t V_t^{-1}\Phi_t^{\top}E_t )} + \sqrt{\Tr(N_t^{\top}\Phi_t V_t^{-1}\Phi_t^{\top}N_t )} + \sqrt{\lambda} \|\mathbf{M}\|_F \label{estimationterms}
\end{align}
The first term on the right hand side of (\ref{estimationterms}) can be bounded using Theorem \ref{selfnormalized} since $e_t$ is $\| C \sig  C^\top + \sigma_z^2 I \|$-sub-Gaussian vector. Therefore,
\begin{equation} \label{esti_first}
    \sqrt{\Tr(E_t^{\top}\Phi_t V_t^{-1}\Phi_t^{\top}E_t ) } \leq \sqrt{m\| C \sig  C^\top \!\!\!+\! \sigma_z^2 I \| \log \! \left(\!\frac{\operatorname{det}\left(V_t\right)^{1 / 2}}{\delta \operatorname{det}(V)^{1 / 2} }\!\!\right)}
\end{equation}
For the second term, \begin{align*}
    \sqrt{\Tr(N_t^{\top}\Phi_t V_t^{-1}\Phi_t^{\top}N_t )} \leq \frac{1}{\sqrt{\lambda}} \|N_t^\top \Phi_t \|_F &\leq \sqrt{\frac{m}{\lambda}} \left \|\sum_{i=H}^t \phi_i (C \Bar{A}^H\!x_{i-H})^\top  \right\| \\
    &\leq t \sqrt{\frac{m}{\lambda}}  \max_{i\leq t} \left\|\phi_i (C \Bar{A}^H\!x_{i-H})^\top  \right\| \\
    &\leq t \sqrt{\frac{m}{\lambda}} \|C\| \upsilon^H \max_{i\leq t} \|\phi_i \| \|x_{i-H}\|
\end{align*}
During warm-up period, from Lemma D.1 of \citet{lale2020regret}, we have that for all $1\leq t \leq \Tw$, with probability $1-\delta/2$,
\begin{align}
    \label{exploration norms first}
    \| x_t \| &\leq X_{w} \coloneqq \frac{(\sigma_w + \sigma_u \|B\|)  \Phi(A) \rho(A)}{\sqrt{1-\rho(A)^2}}\sqrt{2n\log(12n\Tw/\delta)} , \\
    \| z_t \| &\leq Z \coloneqq  \sigma_z \sqrt{2m\log(12m\Tw/\delta)} , \\ 
    \| u_t \| &\leq U_{w} \coloneqq \sigma_u \sqrt{2p\log(12p\Tw/\delta)}, \\
    \| y_t \| &\leq \|C\| X_w + Z
    \label{exploration norms last}.
\end{align}
Thus, during the warm-up phase, we have $\max_{i\leq t\leq \Tw} \|\phi_i\|\|x_{i-H}\| \leq  \Upsilon_w X_w \sqrt{H} $, where $\Upsilon_w = \|C\| X_w + Z + U_w $. During the adaptive control phase, from Lemma \ref{Boundedness}, we have that for all $t \geq \Tw$, with probability $1-2\delta$, 
\begin{align}
    \label{adaptive control norms first}
    \| x_t \| &\leq X_{ac}\coloneqq \|\Sigma\|^{1/2}\sqrt{2n\log(2nT/\delta)} + \bar{\Delta} + \tilde{\mathcal{X}} , \\
    \| u_t \| &\leq \Gamma \tilde{\mathcal{X}} , \\
    \| y_t \| &\leq \tilde{\mathcal{Y}}
    \label{adaptive control norms last}.
\end{align}
Thus, after the warm-up phase, we have $\max_{\Tw \leq t\leq T} \|\phi_i\|\|x_{i-H}\| \leq \Upsilon_c X_{ac} \sqrt{H}$, where $\Upsilon_c = \tilde{\mathcal{Y}} + \Gamma \tilde{\mathcal{X}} $. Therefore for all $t$,
\begin{align*}
    \sqrt{\Tr(N_t^{\top}\Phi_t V_t^{-1}\Phi_t^{\top}N_t )} \leq t \sqrt{\frac{mH}{\lambda}} \|C\| \upsilon^H \max \left\{ \Upsilon_c X_{ac} , \Upsilon_w X_w \right\}
\end{align*}
Picking $H = \frac{2\log(T) + \log(\max \left\{ \Upsilon_c X_{ac} , \Upsilon_w X_w \right\}) + 0.5\log (m/ \lambda ) + \log(\|C\|)}{\log(1/\upsilon)}$ gives \begin{align} \label{esti_second}
    \sqrt{\Tr(N_t^{\top}\Phi_t V_t^{-1}\Phi_t^{\top}N_t )} \leq \frac{t}{T^2} \sqrt{H} 
\end{align}

Combining (\ref{esti_first}) and (\ref{esti_second}) gives the statement of Theorem \ref{theo:closedloopid}. 
\null\hfill$\square$

\noindent \textbf{Proof of Theorem \ref{theo:id2norm}:}
For $\| \mathbf{M}\|_F \leq S$, we have 
\begin{align*}
    \sigma_{\min}(V_t)\|\mathbf{\hat{M}_t} \!-\! \mathbf{M}\|^2_F &\leq \Tr((\mathbf{\hat{M}_t} \!-\! \mathbf{M})V_t(\mathbf{\hat{M}_t}\!-\!\mathbf{M})^{\top}) \\
    &\leq \left(\sqrt{m\| C \sig  C^\top + \sigma_z^2 I \| \log  \left(\frac{\operatorname{det}\left(V_t\right)^{1 / 2}}{\delta \operatorname{det}(V)^{1 / 2} }\right)} \!+\! S\sqrt{\lambda} \!+\!\frac{t\sqrt{H}}{T^2} \right)^2 
\end{align*}

During the warm-up period, for $t\geq T_o$, using Lemma \ref{openloop_persistence_appendix}, we get 
\begin{align}
    \|\mathbf{\hat{M}_0} - \mathbf{M}\|_F &\leq \frac{\sqrt{m\| C \sig  C^\top + \sigma_z^2 I \|\left( \log(\frac{1}{\delta}) + \frac{H(m+p)}{2} \log  \left(\frac{\lambda(m+p) + t\Upsilon_w^2}{\lambda(m+p) }\right)\right)} + S\sqrt{\lambda} +\frac{t\sqrt{H}}{T^2}}{\sqrt{t \frac{\sigma_{o}^2 \min \{ \sigma_w^2, \sigma_z^2, \sigma_u^2 \}}{2}}} \leq \frac{R_{\text{warm}}}{\sqrt{t}}. \nonumber
\end{align}
where $R_{\text{warm}} = \frac{\sqrt{2m\| C \sig  C^\top + \sigma_z^2 I \|\left( \log(1/\delta) + \frac{H(m+p)}{2} \log  \left(\frac{\lambda(m+p) + \Tw \Upsilon_w^2}{\lambda(m+p) }\right)\right)} + S\sqrt{2\lambda} +\frac{\sqrt{2H}}{T}}{\sigma_{o} \min \{ \sigma_w, \sigma_z, \sigma_u \}}$. Let $T_{\mathbf{M}} = R_{\text{warm}}^2$. For $\Tw \geq T_{\mathbf{M}}$, we will have $\|\mathbf{\hat{M}_0} - \mathbf{M}\|_F \leq 1$.

During the adaptive control period, for $t\geq T_c+\Tw$, using Lemma \ref{closedloop_persistence_appendix}, we get  
\begin{align}
    \|\mathbf{\hat{M}_t} - \mathbf{M}\|_F &\leq \frac{\sqrt{m\| C \sig  C^\top + \sigma_z^2 I \|\left( \log(1/\delta) + \frac{H(m+p)}{2} \log  \left(\frac{\lambda(m+p) + t\max\{\Upsilon_w^2, \Upsilon_c^2\}}{\lambda(m+p) }\right)\right)} + S\sqrt{\lambda} +\frac{t\sqrt{H}}{T^2}}{\sqrt{\Tw \frac{\sigma_{o}^2 \min \{ \sigma_w^2, \sigma_z^2, \sigma_u^2 \}}{2} + (t-\Tw) \frac{\sigma_{c}^2 \min \{ \sigma_w^2, \sigma_z^2\}}{16}}} \nonumber \\
    &\leq \frac{\sqrt{m\| C \sig  C^\top + \sigma_z^2 I \|\left( \log(1/\delta) + \frac{H(m+p)}{2} \log  \left(\frac{\lambda(m+p) + T\max\{\Upsilon_w^2, \Upsilon_c^2\}}{\lambda(m+p) }\right)\right)} + S\sqrt{\lambda} +\frac{\sqrt{H}}{T}}{\sqrt{t} \sqrt{\min \left \{\frac{\sigma_{o}^2 \min \{ \sigma_w^2, \sigma_z^2, \sigma_u^2 \}}{2} , \frac{\sigma_{c}^2 \min \{ \sigma_w^2, \sigma_z^2\}}{16} \right\}} } \nonumber
\end{align}
\null\hfill$\square$

\section{Confidence Set Construction for System Parameters}
\label{SupSec:ConfSet}

\begin{algorithm}[tbh] 
 \caption{\Sys}
  \begin{algorithmic}[1] 
  \STATE {\bfseries Input:} $\mathbf{\hat{M}_{t}}$, $H$, system order $n$, $d_1, d_2$ such that $d_1 + d_2 + 1 = H$ \\
  \STATE Form two $d_1 \times (d_2+1)$ Hankel matrices $\mathbf{\mathcal{H}_{\hat{F}_t}}$ and $\mathbf{\mathcal{H}_{\hat{G}_t}}$  from $\mathbf{\hat{M}_{t}}$ and construct $\hat{\mathcal{H}}_t = \left[ \mathbf{\mathcal{H}_{\hat{F}_t}}, \enskip \mathbf{\mathcal{H}_{\hat{G}_t}} \right] \in \R^{md_1 \times (m+p)(d_2+1)}$
  \STATE Obtain $\hat{\mathcal{H}}_t^-$ by discarding $(d_2 + 1)$th and $(2d_2 + 2)$th block columns of $\hat{\mathcal{H}}_t$
  \STATE Using SVD obtain $\hat{\mathcal{N}}_t \in \R^{m d_1 \times (m+p) d_2}$, the best rank-$n$ approximation of $\hat{\mathcal{H}}_t^-$
  \STATE Obtain  $\mathbf{U_t},\mathbf{\Sigma_t},\mathbf{V_t} = \text{SVD}(\hat{\mathcal{N}}_t)$\STATE Construct $\mathbf{\hat{O}_t}(\bar{A},C,d_1) = \mathbf{U_t}\mathbf{\Sigma_t}^{1/2} \in \R^{md_1 \times n}$
  \STATE Construct $[\mathbf{\hat{C}_t}(\bar{A},F,d_2+1), \enskip \mathbf{\hat{C}_t}(\bar{A},B,d_2+1)] = \mathbf{\Sigma_t}^{1/2}\mathbf{V_t} \in \R^{n \times (m+p)d_2}$
  \STATE Obtain $\hat{C}_t\in \R^{m\times n}$, the first $m$ rows of $\mathbf{\hat{O}_t}(\bar{A},C,d_1)$
  \STATE Obtain $\hat{B}_t\in \R^{n\times p}$, the first $p$ columns of $\mathbf{\hat{C}_t}(\bar{A},B,d_2+1)$
  \STATE Obtain $\hat{F}_t\in \R^{n\times m}$, the first $m$ columns of $\mathbf{\hat{C}_t}(\bar{A},F,d_2+1)$
  \STATE Obtain $\hat{\mathcal{H}}_t^+$ by discarding $1$st and $(d_2+2)$th block columns of $\hat{\mathcal{H}}_t$
  \STATE Obtain $\hat{\Bar{A}}_t = \mathbf{\hat{O}_t}^\dagger(\bar{A},C,d_1) \enskip \hat{\mathcal{H}}_t^+ \enskip [\mathbf{\hat{C}_t}(\bar{A},F,d_2+1), \quad \mathbf{\hat{C}_t}(\bar{A},B,d_2+1)]^\dagger$
  \STATE Obtain $\hat{A}_t = \hat{\Bar{A}}_t + \hat{F}_t \hat{C}_t$
  \STATE Obtain $\hat{L}_t \in \R^{n\times m}$, as the first $n \times m$ block of $ \hat{A}_t^\dagger\mathbf{\hat{O}_t}^\dagger(\bar{A},C,d_1)\hat{\mathcal{H}}_t^-$
  \end{algorithmic}
 \label{SYSID}  
\end{algorithm}

The results in this section are adapted from \citet{lale2020logarithmic}. After estimating $\mathbf{\hat{M}_{t}}$, we construct confidence sets for the unknown system parameters and use these confidence sets to come up with the optimistic controller to exploit the information gathered. \Alg uses \Sys, a method similar to Ho-Kalman method~\citep{ho1966effective}, to estimate the system parameters from $\mathbf{\hat{M}_{t}}$. The outline of the algorithm is given in the main text and in Algorithm \ref{SYSID}. Note that the system is order $n$ and
minimal in the sense that the system cannot be described by a state-space model of order less than $n$. Thus, without loss of generality, $\sigma_{n}(A) > 0$. The results in this section follow similar steps with \citet{oymak2018non} with similar changes mentioned in \citet{lale2020regret}. The following lemma is from \citet{oymak2018non}, it will be used in proving confidence bounds and we provide it for completeness. 

\begin{lemma}[\citep{oymak2018non}] \label{hokalmanstability lemma}
$\mathcal{H}$, $\hat{\mathcal{H}}_t$ and $\mathcal{N}, \hat{\mathcal{N}}_t$ satisfies the following perturbation bounds,

\begin{align*} 
\max \left\{\left\|\mathcal{H}^+-\hat{\mathcal{H}}_t^+\right\|,\left\|\mathcal{H}^- -\hat{\mathcal{H}}_t^-\right\|\right\} \leq \|\mathcal{H}-\hat{\mathcal{H}}_t\| &\leq \sqrt{\min \left\{d_{1}, d_{2}+1\right\}}\|\mathbf{\hat{M}_{t} } - \mathbf{M}\| \\ \|\mathcal{N}-\hat{\mathcal{N}}_t\| \leq 2\left\|\mathcal{H}^- -\hat{\mathcal{H}}_t^-\right\| &\leq 2 \sqrt{\min \left\{d_{1}, d_{2}\right\}}\|\mathbf{\hat{M}_{t} } - \mathbf{M}\|
\end{align*}
\end{lemma}

The following lemma is a slight modification of Lemma B.1 in \citep{oymak2018non}.

\begin{lemma}[\citep{oymak2018non}]\label{lem:ranknperturb} Suppose $\sigma_{\min}(\mathcal{N}) \geq 2\|\mathcal{N}-\hat{\mathcal{N}}\|$ where $\sigma_{\min }(\mathcal{N})$ is the smallest nonzero singular value (i.e. $n$th largest singular value) of $N$. Let rank n matrices $\mathcal{N}, \hat{\mathcal{N}}$ have singular value decompositions $\mathbf{U} \mathbf{\Sigma} \mathbf{V}^{\top}$ and $\mathbf{\hat{U}} \mathbf{\hat{\Sigma}} \mathbf{\hat{V}}^{\top}$
There exists an $n \times n$ unitary matrix $\mathbf{T}$ so that
\begin{equation*}
\left\|\mathbf{U} \mathbf{\Sigma}^{1/2}-\mathbf{\hat{U}} \mathbf{\hat{\Sigma}}^{1/2} \mathbf{T} \right\|_{F}^{2}+\left\|\mathbf{V} \mathbf{\Sigma}^{1/2}-\mathbf{\hat{V}} \mathbf{\hat{\Sigma}}^{1/2} \mathbf{T} \right\|_{F}^{2} \leq \frac{5n \| \mathcal{N} - \hat{\mathcal{N}}\|^2}{\sigma_n(\mathcal{N}) - \| \mathcal{N} - \hat{\mathcal{N}}\| }
\end{equation*}

\end{lemma}

The following is the proof of Theorem \ref{ConfidenceSets}. 
\vspace{0.4cm}

\noindent \textbf{Proof of Theorem \ref{ConfidenceSets}:}
For brevity, we have the following notation $\mathbf{O} = \mathbf{O}(\bar{A},C,d_1)$,  $\mathbf{C_F} = \mathbf{C}(\bar{A},F,d_2+1)$, $\mathbf{C_B} = \mathbf{C}(\bar{A},B,d_2+1)$, 
$\mathbf{\hat{O}_t} = \mathbf{\hat{O}_t}(\bar{A},C,d_1)$, $\mathbf{\hat{C}_{F_t}} = \mathbf{\hat{C}_t}(\bar{A},F,d_2+1)$,
$\mathbf{\hat{C}_{B_t}} = \mathbf{\hat{C}_t}(\bar{A},B,d_2+1)$. Let $T_N = T_{\mathbf{M}} \frac{8H}{\sigma_n^2(\mathcal{H})}$. 
Directly applying Lemma \ref{lem:ranknperturb} with the condition that for given $\Tw \geq T_N$, $\sigma_{\min}(\mathcal{N}) \geq 2\|\mathcal{N}-\hat{\mathcal{N}}\|$, we can guarantee that there exists a unitary transform $\mathbf{T}$ such that 
\begin{equation} \label{boundonSVD}
    \left\|\mathbf{\hat{O}_t} - \mathbf{O}\mathbf{T}  \right\|_F^2 + \left\|[\mathbf{\hat{C}_{F_t}} \enskip \mathbf{\hat{C}_{B_t}}] -  \mathbf{T}^\top  [\mathbf{C_F} \enskip \mathbf{C_B}] \right\|_F^2 \leq \frac{10n \| \mathcal{N} - \hat{\mathcal{N}}_t\|^2}{\sigma_n(\mathcal{N}) }
\end{equation}
Since $\hat{C}_t - \bar{C}\mathbf{T}$ is a submatrix of $\mathbf{\hat{O}_t} - \mathbf{O}\mathbf{T}$, $\hat{B}_t - \mathbf{T}^\top \bar{B}$ is a submatrix of $\mathbf{\hat{C}_{B_t}} - \mathbf{T}^\top \mathbf{C_{B}}$ and $\hat{F}_t - \mathbf{T}^\top \bar{F}$ is a submatrix of $\mathbf{\hat{C}_{F_t}} - \mathbf{T}^\top \mathbf{C_{F}}$, we get the same bounds for them stated in (\ref{boundonSVD}). Using Lemma \ref{hokalmanstability lemma}, with the choice of $d_1,d_2 \geq \frac{H}{2}$, we have 
\begin{equation*}
    \| \mathcal{N} - \hat{\mathcal{N}}_t\| \leq \sqrt{2H}\|\mathbf{\hat{M}_{t} } - \mathbf{M}\|.
\end{equation*}
This provides the advertised bounds in the theorem:
\begin{align*}
    \|\hat{B}_t - \mathbf{T}^\top \bar{B} \|, \|\hat{C}_t - \bar{C}\mathbf{T}\|, \|\hat{F}_t - \mathbf{T}^\top \bar{F}\|  \leq \frac{\sqrt{20nH} \|\mathbf{\hat{M}_{t} } - \mathbf{M}\| }{\sqrt{\sigma_n(\mathcal{N})}}
\end{align*}

Let $T_B = T_{\mathbf{M}} \frac{20nH}{\sigma_n(\mathcal{H})}$. Notice that for $\Tw \geq T_B$, we have all the terms above to be bounded by 1. In order to determine the closeness of $\hat{A}_t$ and $\Bar{A}$ we first consider the closeness of $\hat{\Bar{A}}_t - \mathbf{T}^\top \Bar{\Bar{A}}\mathbf{T}$, where $\Bar{\Bar{A}}$ is the output obtained by Ho-Kalman for $\Bar{A}$ when the input is $\mathbf{M}$. Let $X = \mathbf{O}\mathbf{T}$ and $Y = \mathbf{T}^\top  [\mathbf{C_F} \enskip \mathbf{C_B}]$. Thus, we have
\begin{align*}
    \|\hat{\Bar{A}}_t - \mathbf{T}^\top \Bar{\Bar{A}}\mathbf{T} \|_F &= \| \mathbf{\hat{O}_t}^\dagger \hat{\mathcal{H}}_t^+ [\mathbf{\hat{C}_{F_t}} \enskip \mathbf{\hat{C}_{B_t}}]^\dagger - X^\dagger \mathcal{H}^+ Y^\dagger \|_F \\
    &\leq \left \| \left( \mathbf{\hat{O}_t}^\dagger - X^\dagger \right) \hat{\mathcal{H}}_t^+ [\mathbf{\hat{C}_{F_t}} \enskip \mathbf{\hat{C}_{B_t}}]^\dagger \right \|_F + \left \|X^\dagger \left(\hat{\mathcal{H}}_t^+ - \mathcal{H}^+ \right)[\mathbf{\hat{C}_{F_t}} \enskip \mathbf{\hat{C}_{B_t}}]^\dagger \right \|_F\\
    &\quad + \left \|X^\dagger \mathcal{H}^+ \left([\mathbf{\hat{C}_{F_t}} \enskip \mathbf{\hat{C}_{B_t}}]^\dagger - Y^\dagger \right) \right \|_F
\end{align*}

For the first term we have the following perturbation bound \citep{meng2010optimal,wedin1973perturbation},
\begin{align*}
    \|\mathbf{\hat{O}_t}^\dagger - X^\dagger \|_F &\leq \| \mathbf{\hat{O}_t} - X \|_F \max \{ \|X^\dagger \|^2, \| \mathbf{\hat{O}_t}^\dagger \|^2  \} \leq \| \mathcal{N} - \hat{\mathcal{N}}_t\| \sqrt{\frac{10n}{\sigma_n(\mathcal{N}) }} \max \{ \|X^\dagger \|^2, \| \mathbf{\hat{O}_t}^\dagger \|^2  \}
\end{align*}
Since we have $\sigma_{n}(\mathcal{N}) \geq 2\|\mathcal{N}-\hat{\mathcal{N}}\|$, we have $\|\hat{\mathcal{N}}\| \leq 2 \|\mathcal{N}\| $ and $2\sigma_{n}(\hat{\mathcal{N}}) \geq \sigma_{n}(\mathcal{N})$. Thus,
\begin{equation} \label{daggersingular}
    \max \{ \|X^\dagger \|^2, \| \mathbf{\hat{O}_t}^\dagger \|^2  \} = \max \left\{ \frac{1}{\sigma_{n}(\mathcal{N})}, \enskip \frac{1}{\sigma_{n}(\hat{\mathcal{N}})}  \right\} \leq \frac{2}{\sigma_{n}(\mathcal{N})}
\end{equation}
Combining these and following the same steps for $\|[\mathbf{\hat{C}_{F_t}} ~ \mathbf{\hat{C}_{B_t}}]^\dagger \!-\! Y^\dagger \|_F$, we get 
\begin{equation} \label{perturbationbounds}
    \left \|\mathbf{\hat{O}_t}^\dagger - X^\dagger \right \|_F,\enskip \left \|[\mathbf{\hat{C}_{F_t}} ~ \mathbf{\hat{C}_{B_t}}]^\dagger \!-\! Y^\dagger \right \|_F\leq \left \| \mathcal{N} - \hat{\mathcal{N}}_t \right \| \sqrt{\frac{40n}{\sigma_n^{3}(\mathcal{N}) }} 
\end{equation}
The following individual bounds obtained by using (\ref{daggersingular}), (\ref{perturbationbounds}) and triangle inequality: 
\begin{align*}
    \left \| \left( \mathbf{\hat{O}_t}^\dagger - X^\dagger \right) \hat{\mathcal{H}}_t^+ [\mathbf{\hat{C}_{F_t}} \enskip \mathbf{\hat{C}_{B_t}}]^\dagger \right \|_F &\leq \| \mathbf{\hat{O}_t}^\dagger - X^\dagger \|_F \|\hat{\mathcal{H}}_t^+ \| \|[\mathbf{\hat{C}_{F_t}} \enskip \mathbf{\hat{C}_{B_t}}]^\dagger \| \\
    &\leq \frac{4\sqrt{5n} \left \| \mathcal{N} - \hat{\mathcal{N}}_t \right \|}{\sigma_n^{2}(\mathcal{N})} \left( \|\mathcal{H}^+ \| + \|\hat{\mathcal{H}}_t^+ - \mathcal{H}^+ \| \right)  \\
    \left \|X^\dagger \left(\hat{\mathcal{H}}_t^+ - \mathcal{H}^+ \right)[\mathbf{\hat{C}_{F_t}} \enskip \mathbf{\hat{C}_{B_t}}]^\dagger \right \|_F &\leq \frac{2\sqrt{n}\|\hat{\mathcal{H}}_t^+ - \mathcal{H}^+ \|}{\sigma_n(\mathcal{N})}
    \\
    \left \|X^\dagger \mathcal{H}^+ \left([\mathbf{\hat{C}_{F_t}} \enskip \mathbf{\hat{C}_{B_t}}]^\dagger - Y^\dagger \right) \right \|_F &\leq \|X^\dagger \| \|\mathcal{H}^+ \| \|[\mathbf{\hat{C}_{F_t}} \enskip \mathbf{\hat{C}_{B_t}}]^\dagger - Y^\dagger \| \\
    &\leq \frac{2\sqrt{10n} \left \| \mathcal{N} - \hat{\mathcal{N}}_t \right \|}{\sigma_n^{2}(\mathcal{N})}  \|\mathcal{H}^+ \|
\end{align*}
Combining these we get 
\begin{align*}
    \|\hat{\Bar{A}}_t - \mathbf{T}^\top \Bar{\Bar{A}}\mathbf{T} \|_F &\leq \frac{31\sqrt{n}  \|\mathcal{H}^+ \|  \left \| \mathcal{N} - \hat{\mathcal{N}}_t \right \|}{2\sigma_n^{2}(\mathcal{N})}  + \|\hat{\mathcal{H}}_t^+ - \mathcal{H}^+ \| \left( \frac{4\sqrt{5n} \left \| \mathcal{N} - \hat{\mathcal{N}}_t \right \|}{\sigma_n^{2}(\mathcal{N})} + \frac{2\sqrt{n}}{\sigma_n(\mathcal{N})}  \right) \\
    &\leq \frac{31\sqrt{n}  \|\mathcal{H}^+ \|  }{2\sigma_n^{2}(\mathcal{N})} \left \| \mathcal{N} - \hat{\mathcal{N}}_t \right \| +  \frac{13 \sqrt{n}}{2\sigma_n(\mathcal{N})} \|\hat{\mathcal{H}}_t^+ - \mathcal{H}^+ \|
\end{align*}
Now consider $\hat{A}_t = \hat{\Bar{A}}_t + \hat{F}_t\hat{C}_t$. Using Lemma \ref{hokalmanstability lemma},
\begin{align*}
    &\|\hat{A}_t -  \mathbf{T}^\top \Bar{A}\mathbf{T} \|_F \\
    &= \|\hat{\Bar{A}}_t + \hat{F}_t\hat{C}_t - \mathbf{T}^\top \Bar{\Bar{A}}\mathbf{T} - \mathbf{T}^\top \bar{F} \bar{C} \mathbf{T} \|_F \\
    &\leq \|\hat{\Bar{A}}_t - \mathbf{T}^\top \Bar{\Bar{A}}\mathbf{T} \|_F + \|(\hat{F}_t -\mathbf{T}^\top \bar{F}) \hat{C}_t  \|_F + \|\mathbf{T}^\top \bar{F} (\hat{C}_t - \bar{C} \mathbf{T}) \|_F \\
    &\leq \|\hat{\Bar{A}}_t - \mathbf{T}^\top \Bar{\Bar{A}}\mathbf{T} \|_F + \|(\hat{F}_t -\mathbf{T}^\top \bar{F})\|_F \| \hat{C}_t -\bar{C} \mathbf{T} \|_F + \|(\hat{F}_t -\mathbf{T}^\top \bar{F})\|_F \| \bar{C} \| + \|\bar{F}\| \|(\hat{C}_t - \bar{C} \mathbf{T}) \|_F \\
    &\leq \frac{31\sqrt{n}  \|\mathcal{H}^+ \|  }{2\sigma_n^{2}(\mathcal{N})} \! \left \| \mathcal{N} \!-\! \hat{\mathcal{N}}_t \right \| \!+\!  \frac{13 \sqrt{n}}{2\sigma_n(\mathcal{N})} \|\hat{\mathcal{H}}_t^+ \!-\! \mathcal{H}^+ \| \!+\! \frac{10n \| \mathcal{N} \!-\! \hat{\mathcal{N}}_t\|^2}{\sigma_n(\mathcal{N}) } \!+\! (\|\bar{F} \| \!+\! \|\bar{C} \|)\| \mathcal{N} \!-\! \hat{\mathcal{N}}_t\| \sqrt{\frac{10n }{\sigma_n(\mathcal{N}) }  } \\
    &\leq \frac{31\sqrt{2nH}  \|\mathcal{H} \|  }{2\sigma_n^{2}(\mathcal{N})} \! \|\mathbf{\hat{M}_{t} } - \mathbf{M} \| + \frac{13 \sqrt{nH}}{2\sqrt{2}\sigma_n(\mathcal{N})} \|\mathbf{\hat{M}_{t} } - \mathbf{M} \| + \frac{20nH \|\mathbf{\hat{M}_{t} } - \mathbf{M} \|^2}{\sigma_n(\mathcal{N}) } \\
    &\qquad + (\|\bar{F} \| \!+\! \|\bar{C} \|)\|\mathbf{\hat{M}_{t} } - \mathbf{M} \| \sqrt{\frac{20nH }{\sigma_n(\mathcal{N}) }  }
\end{align*}
Define $T_A$ such that 
\begin{equation}
    T_A = T_{\mathbf{M}} \left( \frac{\frac{62\sqrt{2nH}  \|\mathcal{H} \|  }{2\sigma_n^{2}(\mathcal{N})} + \frac{26 \sqrt{nH}}{2\sqrt{2}\sigma_n(\mathcal{N})} + (\|\bar{F} \| \!+\! \|\bar{C} \|)\sqrt{\frac{80nH }{\sigma_n(\mathcal{N})}} + \sqrt{\frac{40nH\sigma_n(\bar{A})}{\sigma_n(\mathcal{N}) }}}{\sigma_n(\bar{A})} \right)^2.
\end{equation}
Notice that for $\Tw \geq T_A$, we have $\|\hat{A}_t -  \mathbf{T}^\top \Bar{A}\mathbf{T} \| \leq \sigma_{n}(\Bar{A})/2$. Since $\Tw \geq T_A$, from Weyl's inequality we have $\sigma_n(\hat{A}_t) \geq  \sigma_{n}(\Bar{A})/2$. Recalling that $X = \mathbf{O}(\bar{A},C,d_1)\mathbf{T}$, under Assumption \ref{AssumContObs} we consider $\hat{L}_t$: 
\begin{align*}
    &\|\hat{L}_t - \mathbf{T}^\top \bar{L} \|_F \\
    &= \| \hat{A}_t^\dagger \mathbf{\hat{O}_t}^\dagger \hat{\mathcal{H}}_t^- - \mathbf{T}^\top \bar{A}^\dagger \mathbf{O}^\dagger \mathcal{H}^- \|_F \\
    &\leq \|(\hat{A}_t^\dagger - \mathbf{T}^\top \bar{A}^\dagger \mathbf{T}) \mathbf{\hat{O}_t}^\dagger \hat{\mathcal{H}}_t^-  \|_F +  \|\mathbf{T}^\top \bar{A}^\dagger \mathbf{T} (\mathbf{\hat{O}_t}^\dagger - X^\dagger)\hat{\mathcal{H}}_t^- \|_F + \|\mathbf{T}^\top \bar{A}^\dagger \mathbf{T} X^\dagger (\hat{\mathcal{H}}_t^- - \mathcal{H}^- ) \|_F \\
    &\leq \|\hat{A}_t^\dagger - \mathbf{T}^\top \bar{A}^\dagger \mathbf{T} \|_F \|\mathbf{\hat{O}_t}^\dagger \| \|\hat{\mathcal{H}}_t^- \| +\| \mathbf{\hat{O}_t}^\dagger - X^\dagger \|_F \|\bar{A}^\dagger \| \| \hat{\mathcal{H}}_t^- \| + \sqrt{n} \|\hat{\mathcal{H}}_t^- - \mathcal{H}^- \| \|\bar{A}^\dagger \| \|X^\dagger \| \\
    &\leq \left( \|\hat{A}_t^\dagger \!-\! \mathbf{T}^\top \bar{A}^\dagger \mathbf{T} \|_F \sqrt{\frac{2}{\sigma_{n}(\mathcal{N})}} + \left \| \mathcal{N} \!-\! \hat{\mathcal{N}}_t \right \| \sqrt{\frac{40n}{\sigma_n^{3}(\mathcal{N}) }} \| \bar{A}^\dagger \| \right) \left( \|\mathcal{H}^- \| \!+\! \|\hat{\mathcal{H}}_t^- \!-\! \mathcal{H}^- \| \right) \\
    &\qquad + \sqrt{n} \| \bar{A}^\dagger \| \frac{1}{\sqrt{\sigma_n(\mathcal{N})}} \|\hat{\mathcal{H}}_t^- - \mathcal{H}^- \|
\end{align*}

Again using the perturbation bounds of the Moore–Penrose inverse under the Frobenius norm \citep{meng2010optimal}, we have $\|\hat{A}_t^\dagger - \mathbf{T}^\top \bar{A}^\dagger \mathbf{T} \|_F \leq \frac{2}{\sigma_n^2(\Bar{A})} \|\hat{A}_t -  \mathbf{T}^\top \Bar{A}\mathbf{T} \| $. Notice that the similarity transformation that transfers $A$ to $\Bar{A}$ is bounded since $S = \left([ C^\top~(C\Bar{A}) ^\top    \ldots (C\Bar{A}^{d_1-1}) ^\top] ^\top \right)^\dagger  \mathbf{O}(\bar{A},C,d_1) $. Combining all and using Lemma \ref{hokalmanstability lemma}, we obtain the confidence set for $\hat{L}_t$ given in Theorem \ref{ConfidenceSets}. 

\null\hfill$\square$

\section{Boundedness of The Output and State Estimation, Proof of Lemma \ref{Boundedness} } \label{SupSec:Boundedness}

The proof of Lemma \ref{Boundedness} follows similar arguments with the proof of Lemma 4.2 of \citep{lale2020regret}. The main difference is that \Alg, the system estimations are refined during the adaptive control period, thus the control policy is refined. Also, since the behavior of a system and its similarity transformation is the same, without loss of generality we assume that similarity transformation $\mathbf{T} = I$.
\\
\noindent \textbf{Proof of Lemma \ref{Boundedness}:}

Assume that $\Theta \in (\mathcal{C}_A(t) \times \mathcal{C}_B(t) \times \mathcal{C}_C(t) \times \mathcal{C}_L(t)) $ for all $t\geq \Tw$, which is holds with probability $1-\delta$. We can write the decomposition for $\hat{x}_{t|t,\ttt}$ as follows,
\begin{align}
    \hat{x}_{t|t,\ttt} &= \hat{x}_{t|t-1,\ttt} + \tl_t (y_t - \tc_t \hat{x}_{t|t-1,\ttt}) \nonumber \\
    &= \ta_{t-1} \hat{x}_{t-1|t-1,\ttt} - \tb_{t-1} \tk_{t-1} \hat{x}_{t-1|t-1,\ttt} + \tl_t (y_t - \tc_t (\ta_{t-1}\hat{x}_{t-1|t-1,\ttt} - \tb_{t-1} \tk_{t-1} \hat{x}_{t-1|t-1,\ttt})) \nonumber \\
    &= (I - \tl_t\tc_{t})(\ta_{t-1} - \tb_{t-1}\tk_{t-1}) \hat{x}_{t-1|t-1,\ttt} + \tl_t y_t \nonumber\\
    &= (I - \tl_t\tc_{t})(\ta_{t-1} - \tb_{t-1}\tk_{t-1}) \hat{x}_{t-1|t-1,\ttt} \nonumber \\
    &\quad \qquad \qquad +\tl_t \left(Cx_t - C\hat{x}_{t|t-1,\ttt} + C\hat{x}_{t|t-1,\ttt}  + z_t\right) \nonumber \\
    &= (I - \tl_t\tc_{t})(\ta_{t-1} - \tb_{t-1}\tk_{t-1}) \hat{x}_{t-1|t-1,\ttt} \nonumber \\
    &\quad \qquad \qquad +\tl_t \left(Cx_t - C\hat{x}_{t|t-1,\ttt} + C(\ta_{t-1} - \tb_{t-1}\tk_{t-1}) \hat{x}_{t-1|t-1,\ttt} + z_t\right) \nonumber \\
    &= \left(\ta_{t-1} - \tb_{t-1}\tk_{t-1} - \tl_t \left(\tc_t \ta_{t-1} -\tc_t \tb_{t-1}\tk_{t-1} - C\ta_{t-1} + C\tb_{t-1}\tk_{t-1} \right) \right)\hat{x}_{t-1|t-1,\ttt} \nonumber  \\
    &\quad \qquad \qquad + \tl_t C(x_{t} - \hat{x}_{t|t-1,\Theta} + \hat{x}_{t|t-1,\Theta} - \hat{x}_{t|t-1,\ttt}) + \tl_t z_t \nonumber  \\
    &= \left(\ta_{t-1}-\tb_{t-1}\tk_{t-1} - \tl_t\left(\tc_t\ta_{t-1} -\tc_t\tb_{t-1}\tk_{t-1} - C\ta_{t-1} + C\tb_{t-1}\tk_{t-1} \right) \right)\hat{x}_{t-1|t-1,\ttt} \nonumber \\
    &\quad \qquad \qquad + \tl_t C(x_{t} - \hat{x}_{t|t-1,\Theta}) + \tl_t C (\hat{x}_{t|t-1,\Theta} - \hat{x}_{t|t-1,\ttt}) + \tl_t z_t. \label{estimation propagate}
\end{align}

Thus, the dynamics of $\hat{x}_{t|t,\ttt}$ is governed by 
\begin{equation*}
\mathbf{N}_t =\ta_{t-1}-\tb_{t-1}\tk_{t-1} - \tl_t\left(\tc_t\ta_{t-1} -\tc_t\tb_{t-1}\tk_{t-1} - C\ta_{t-1} + C\tb_{t-1}\tk_{t-1} \right)    
\end{equation*}
and it is driven by the process of $\tl_t C(x_{t} - \hat{x}_{t|t-1,\Theta}) + \tl_t C (\hat{x}_{t|t-1,\Theta} - \hat{x}_{t|t-1,\ttt}) + \tl_t z_t$. Let $T_{u} =  T_B \left(\frac{ 2\zeta \rho }{1-\rho} \right)^2$. With the Assumption~\ref{Stabilizable set}, and for $\Tw \geq T_u$, we have that $\|\tc_t - C  \| \leq \frac{1-\rho}{2\zeta\rho}$ which gives $\|\mathbf{N}_t\| \leq \frac{1+\rho}{2} < 1$ for all $t\geq\Tw$. Similar to the proof of Lemma 4.2 in \citep{lale2020regret}, we have that $\tl_t C (x_{t} - \hat{x}_{t|t-1,\Theta}) + \tl_t z_t$ is $\zeta (\| C\| \|\Sigma\|^{1/2} + \sigma_z)$-sub-Gaussian, thus it's $\ell_2$-norm can be bounded using Lemma \ref{subgauss lemma}:
\begin{equation*}
    \|\tl_t C (x_{t} - \hat{x}_{t|t-1,\Theta}) + \tl_t z_t\| \leq \zeta (\| C\| \|\Sigma\|^{1/2} + \sigma_z) \sqrt{2n\log(2nT/\delta)}
\end{equation*}
for all $t \geq \Tw$ with probability at least $1-\delta$. A special care is needed for $\hat{x}_{t|t-1,\Theta} - \hat{x}_{t|t-1,\ttt}$. Denote $\Delta_t = \hat{x}_{t|t-1,\Theta} - \hat{x}_{t|t-1,\ttt}$. Consider the decomposition given in equation (51) in \citet{lale2020regret}. In this setting, since at each time step after the warm-up, the estimation errors are monotonically decreasing, therefore we can upper bound the norm of each term in the decomposition by the norm of the term at the time of end of warm-up. Let 
\begin{align}
    &T_\alpha \!=\! T_B \left( \frac{\Gamma\left(1+\zeta(1+\| C\|)\right)}{\sigma - \upsilon}\right)^2\!\!, \qquad T_\gamma = T_A \frac{\sigma_n^2(\bar{A})}{4} \left( \frac{1 + \Gamma(1+\zeta\| B\|)}{\sigma - \rho} \right)^2 ,  \nonumber \\
    &T_{\beta} \!=\! T_A \frac{\sigma_n^2(\bar{A})}{4}\!\! \left( \frac{\Gamma\|B\|(1\!+\!\zeta \!+\! \zeta \|C\|)(\Phi(A)\zeta \!+\! (1\!+\!\Gamma)(1\!+\!\zeta) ) }{(1 - \sigma)^2} \right)^2\!\!\!.
    \label{times}
\end{align}

Thus, using the arguments in \cite{lale2020regret}, we can show that after a warm-up period of $\Tw \geq \max \{T_{\alpha}, T_{\gamma}\}$, we have that for all $t\geq \Tw$, $\max \{ \|(A \!+\! (\ta_t - A - \tb_t \tk_t + B \tk_t)) (I \!-\!\tl_t \tc_t) \|, \|A \!-\! B\tk_{t} \!+\! B\tk_{t}\tl_{t}(\tc_t\!-\!C) \| \} \leq \sigma < 1 $. Using the inductive argument given in \citep{lale2020regret}, we can show that for all $t\geq \Tw \geq T_{\beta}$, $\|\Delta_t \| \leq \bar{\Delta} $ with probability $1-\delta$. Notice that the definition of $\bar{\Delta}$ still includes the same terms given in equation (54) of \citet{lale2020regret} but $\beta_A, \beta_B, \beta_C$ is replaced with $\beta_A(\Tw), \beta_B(\Tw), \beta_C(\Tw)$ and $\Delta L$ is replaced by $2\beta_L(\Tw)$ due to new estimation method, \textit{i.e.},
\begin{align*}
    \bar{\Delta} &= 10\left(\frac{\bar{\kappa} }{1-\sigma} + \frac{\bar{\beta} \bar{\xi} }{(1-\sigma)^2} \right)\left(\| C\| \|\Sigma\|^{1/2} + \sigma_z\right) \sqrt{2m\log(2mT/\delta)} 
\end{align*}
for $\bar{\kappa} = 2\Phi(A)\beta_L(\Tw) + 2\zeta (\beta_A(\Tw) + \Gamma \beta_B(\Tw))$, $\bar{\beta} = 2\zeta \beta_C(\Tw) (\Phi(A) + 2(\beta_A(\Tw) + \Gamma \beta_B(\Tw))) + 2(\beta_A(\Tw) + \Gamma \beta_B(\Tw)) $ and $\bar{\xi} = \zeta (\rho + 2(\beta_A(\Tw) + \Gamma \beta_B(\Tw))) + 2\|B\|\Gamma \beta_L(\Tw) $. Thus, we get 
\begin{align}
\| \hat{x}_{t|t,\ttt}\| &= \left \| \sum_{i=1}^t \mathbf{N}^{t-i} \left( \tl_i C(x_{i-1} - \hat{x}_{i|i-1,\Theta}) + \tl_i C (\hat{x}_{i|i-1,\Theta} - \hat{x}_{i|i-1,\ttt}) + \tl_i z_i \right) \right \| \\
&\leq \max_{1\leq i\leq t}\left\| \tl_i C(x_{i-1} - \hat{x}_{i|i-1,\Theta}) + \tl_i C (\hat{x}_{i|i-1,\Theta} - \hat{x}_{i|i-1,\ttt}) + \tl_i z_i   \right \| \left( \sum_{i=1}^t \|\mathbf{M}\|^{t-i}  \right) \\
&\leq \frac{2}{1-\rho} \max_{1\leq i\leq t}\left\| \tl_i C(x_{i-1} - \hat{x}_{i|i-1,\Theta}) + \tl_i C (\hat{x}_{i|i-1,\Theta} - \hat{x}_{i|i-1,\ttt}) + \tl_i z_i   \right \| \\
&\leq \tilde{\mathcal{X}} \coloneqq \frac{2\zeta \left(\|C\| \bar{\Delta} + \left(\| C\| \|\Sigma\|^{1/2} + \sigma_z\right) \sqrt{2n\log(2nT/\delta)}\right) }{1-\rho}.
\end{align}
with probability $1-2\delta$. For $y_t$, we have the following decomposition,
\begin{align*}
    y_t &= C \hat{x}_{t|t-1,\ttt} + C(x_t - \hat{x}_{t|t-1,\ttt}) + z_{t} \\
    &= C \hat{x}_{t|t-1,\ttt} + C(x_t - \hat{x}_{t|t-1,\Theta}) +C (\hat{x}_{t|t-1,\Theta} -  \hat{x}_{t|t-1,\ttt}) + z_{t} \\
    &= C(\ta_{t-1} - \tb_{t-1} \tk_{t-1}) \hat{x}_{t-1|t-1,\ttt} + C(x_t - \hat{x}_{t|t-1,\Theta}) +C (\hat{x}_{t|t-1,\Theta} -  \hat{x}_{t|t-1,\ttt}) + z_{t}
\end{align*}
Using similar analysis with $\hat{x}_{t|t,\ttt}$, we get the following bound for $y_t$ for all $t\geq \Tw$: 
\begin{align*}
    \|y_t\| \leq \rho \|C\| \tilde{\mathcal{X}} + (\| C\| \|\Sigma\|^{1/2} + \sigma_z) \sqrt{2m\log(2mT/\delta)} + \|C\| \bar{\Delta}     
\end{align*}
with probability $1-2\delta$. Thus, all three statements of Lemma \ref{Boundedness} hold with probability at least $1-3\delta$. 
\null\hfill$\square$
\newpage
\section{Regret Decomposition}
\label{SupSec:RegDecomp}

Recall the following lemma from \citep{lale2020regret} on the Bellman optimality equation for \LQG: 

\begin{lemma}[Bellman Optimality Equation for \LQG \citep{lale2020regret}] \label{LQGBellman}
Given state estimation $\hat{x}_{t|t-1} \in \R^{n}$ and an observation $y_t \in \R^{m}$ pair at time $t$, Bellman optimality equation of average cost per stage control of \LQG system $\Theta = (A,B,C)$ with regulating parameters $Q$ and $R$ is 
\begin{align}
    &J_*(\Theta)  +  \hat{x}_{t|t}^\top  \left( P  -  C^\top Q C\right)  \hat{x}_{t|t}  +  y_t^\top Q y_t =  \min_u   \bigg\{ y_t^\top Q y_t + u^\top   R u \label{bellman} \\
    &\qquad \qquad \qquad \qquad \qquad  \qquad \qquad \qquad +  \mathbb{E}\bigg[ \hat{x}_{t+1|t+1}^{u \top}  \left( P  -   C^\top Q C  \right)  \hat{x}_{t+1|t+1}^{u}    +  y_{t+1}^{u \top} Q y_{t+1}^u  \bigg] \nonumber
    \bigg\}
\end{align}
where $P$ is the unique solution to DARE of $\Theta$, $\hat{x}_{t|t}= (I-LC)\hat{x}_{t|t-1} + L y_t$, $y_{t+1}^u = C(Ax_t + Bu + w_t) + z_{t+1}$, and $\hat{x}_{t+1|t+1}^{u} = \left( I - LC\right)(A\hat{x}_{t|t} + Bu) + Ly_{t+1}^u$.
The equality is achieved by the optimal controller of $\Theta$.
\end{lemma}

Using Lemma \ref{LQGBellman} for the optimistic system at time $t$, we derive the instantaneous regret decomposition at time $t$ with the following expressions:

\begin{align}
    \hat{x}_{t|t,\ttt_t} &= \left( I - \tl_t\tc_t \right)\hat{x}_{t|t-1} + \tl_t y_t \label{firstdef} \\
    y_{t+1,\ttt_t} &= \tc_t \left( \ta_t - \tb_t \tk_t \right) \hat{x}_{t|t,\ttt_t} + \tc_t \ta_t \left( x_t - \hat{x}_{t|t,\ttt_t} \right) + \tc_t w_t + z_{t+1} \\
    \hat{x}_{t+1|t+1,\ttt_t} &= \left( \ta_t - \tb_t \tk_t \right) \hat{x}_{t|t,\ttt} + \tl_t \tc_t \ta_t \left( x_t - \hat{x}_{t|t,\ttt} \right) + \tl_t \tc_t w_t + \tl_t z_{t+1} \\
    y_{t+1, \Theta} &= CA \hat{x}_{t|t,\ttt} - CB\tk_t \hat{x}_{t|t,\ttt} + C w_t + CA(x_t - \hat{x}_{t|t,\ttt}) + z_{t+1} \\
    \hat{x}_{t+1|t+1,\Theta} &= (I -LC)(A\hat{x}_{t|t,\Theta} - B\tk_t\hat{x}_{t|t,\ttt}) + L y_{t+1, \Theta} \\
    &= (I-LC)(A-B\tk_t)\hat{x}_{t|t,\ttt} + (I-LC)A(\hat{x}_{t|t,\Theta} - \hat{x}_{t|t,\ttt}) + L y_{t+1, \Theta} \\
    &= (I-LC)(A-B\tk_t)\hat{x}_{t|t,\ttt} + LC(A-B\tk)\hat{x}_{t|t,\ttt} + LC w_t + LCA(x_t - \hat{x}_{t|t,\ttt}) \nonumber \\
    \qquad \qquad & \qquad + (I-LC)A(\hat{x}_{t|t,\Theta} - \hat{x}_{t|t,\ttt}) + Lz_{t+1} \\
    &= (A\!-\!B\tk_t)\hat{x}_{t|t,\ttt} \!+\! LCw_t \!+\! LCA(x_t \!-\! \hat{x}_{t|t,\ttt_t}) \!+\! (I\!-\!LC)A(\hat{x}_{t|t,\Theta} \!-\! \hat{x}_{t|t,\ttt_t}) \!+\! Lz_{t+1}. \label{lastdef}
\end{align}

Note that these expressions are the time varying counterparts for the same expressions in \citet{lale2020regret}. Thus, the regret decomposition is similar to the regret decomposition derived in \citet{lale2020regret}, but with some changes. Since we are updating the optimistic choices during the adaptive control each regret term is written using the expressions given (\ref{firstdef})-(\ref{lastdef}). This brings the only significant change in term $R_1$ in the regret decomposition of \citet{lale2020regret}. In order to analyze the effect of policy changes and obtain a similar analysis for $R_1$, we obtain these two terms: 
\begin{align*}
     &R_1 = \sum_{t=1}^{T} \left \{ \hat{x}_{t|t,\tth}^\top \left( \tp_t - \tc_t^\top Q \tc_t \right) \hat{x}_{t|t,\tth} \!-\! \mathbb{E}\left[ \hat{x}_{t+1|t+1,\Theta}^\top \left( \tp_{t+1} - \tc_{t+1}^\top Q \tc_{t+1}  \right) \hat{x}_{t+1|t+1,\Theta} \Big | \hat{x}_{t|t-1}, y_t, u_t \right]  \right \} \\
     &R_{\text{update}}= \sum_{t=1}^{T}  \mathbb{E}\left[ \hat{x}_{t+1|t+1,\Theta}^\top \left( (\tp_{t} - \tc_{t}^\top Q \tc_{t}) - (\tp_{t+1} - \tc_{t+1}^\top Q \tc_{t+1})  \right) \hat{x}_{t+1|t+1,\Theta} \Big | \hat{x}_{t|t-1}, y_t, u_t \right]
\end{align*}
Therefore, due to Lemma \ref{Boundedness}, the overall regret decomposition can be represented as 
\begin{align}
    \sum_{t=1}^{T} \left( y_t^\top Q y_t \!+\! u_t^\top R u_t \right) \!&=\!\! \sum_{t=1}^{T}  J_*(\tth) \!+\! R_1 \!+\! R_2 \!-\! R_3 \!-\! R_4 \!-\! R_5 \!-\! R_6 \!-\! R_7 \!-\! R_8 \!-\! R_9 \!-\! R_{10} \!-\! R_{11} - R_{\text{update}} \nonumber \\
    &\leq T J_*(\Theta) \!+\! R_1 \!+\! R_2 \!-\! R_3 \!-\! R_4 \!-\! R_5 \!-\! R_6 \!-\! R_7 \!-\! R_8 \!-\! R_9 \!-\! R_{10} \!-\! R_{11} \!-\! R_{\text{update}} \label{useoptimism}
\end{align}
for 
\begin{align*}
    R_2& \!=\!\! \sum_{t=1}^{T} \left \{ y_t^\top Q y_t - 
    \mathbb{E}\left[ y_{t+1,\Theta}^\top Q y_{t+1,\Theta} \Big | \hat{x}_{t|t-1}, y_t, u_t \right] \right \}, \\
    R_3& \!=\!\! \sum_{t=1}^{T} \left \{ \hat{x}_{t|t,\tth}^\top (\ta_t - \tb_t \tk_t )^\top \tc_t^\top Q \tc_t (\ta_t - \tb_t \tk_t) \hat{x}_{t|t,\tth} - \hat{x}_{t|t,\tth}^\top (A-B\tk_t)^\top C^\top Q C (A-B\tk_t) \hat{x}_{t|t,\tth} \right \}, \\
    R_4& \!=\!\! \sum_{t=1}^{T} \! \left\{ \hat{x}_{t|t,\tth}^\top (\ta_t \!-\! \tb_t \tk_t )^\top (\tp_t \!-\! \tc_t^\top Q \tc_t) (\ta_t \!-\! \tb_t \tk_t) \hat{x}_{t|t,\tth} \!-\! \hat{x}_{t|t,\tth}^\top (A\!-\!B\tk_t)^\top (\tp_t - \tc_t^\top Q \tc_t) (A\!-\!B\tk_t) \hat{x}_{t|t,\tth}
     \right \}, \\
    R_5& \!=\! -\!\! \sum_{t=1}^{T} \left \{ 2\hat{x}_{t|t,\tth}^\top (A-B\tk_t)^\top (\tp_t - \tc_t^\top Q \tc_t) (I-LC)A(\hat{x}_{t|t,\Theta} -\hat{x}_{t|t,\tth} ) \right \}, \\
    R_6& \!=\! -\!\! \sum_{t=1}^{T} \left \{ (\hat{x}_{t|t,\Theta} -\hat{x}_{t|t,\tth} )^\top A^\top (I-LC)^\top (\tp_t - \tc_t^\top Q \tc_t) (I-LC)A(\hat{x}_{t|t,\Theta} -\hat{x}_{t|t,\tth} ) \right \}, \\
    R_7& \!=\!\! \sum_{t=1}^{T} \left\{ \mathbb{E} \left[ w_t^\top \tc_t^\top Q \tc_t w_t   \right] - \mathbb{E}\left[w_t^\top C^\top Q C w_t \right] \right\}, \\
    R_8& \!=\!\! \sum_{t=1}^{T} \left \{ \mathbb{E}  \left[ w_t^\top \tc_t^\top \tl_t^\top  \left( \tp_t - \tc_t^\top Q \tc_t  \right) \tl_t \tc_t w_t   \right]   - \mathbb{E}\left[w_t^\top C^\top L^\top \left( \tp_t - \tc_t^\top Q \tc_t  \right) L C  w_t  \right] \right \}, \\
    R_9& \!=\!\! \sum_{t=1}^{T} \bigg \{ \mathbb{E}  \bigg[ \left(x_t - \hat{x}_{t|t,\tth}\right)^\top \ta_t^\top \tc_t^\top Q \tc_t \ta_t \left(x_t - \hat{x}_{t|t,\tth}\right) \Big | \hat{x}_{t|t-1}, y_t   \bigg] \\ &\qquad \qquad-  \mathbb{E}\left[\left(x_{t}-\hat{x}_{t|t,\tth}\right)^\top A^\top C^\top Q C A \left(x_{t}-\hat{x}_{t|t,\tth}\right)  \Big | \hat{x}_{t|t-1}, y_t \right] \bigg \}, \\
    R_{10}& \!=\!\! \sum_{t=1}^{T} \bigg \{ \mathbb{E} \bigg[ \left(x_t - \hat{x}_{t|t,\tth}\right)^\top \ta_t^\top \tc_t^\top \tl_t^\top  \left( \tp_t - \tc_t^\top Q \tc_t  \right) \tl_t\tc_t\ta_t \left(x_t - \hat{x}_{t|t,\tth}\right) \Big | \hat{x}_{t|t-1}, y_t   \bigg] \\
    &\qquad \qquad - \mathbb{E}\left[\left(x_{t}-\hat{x}_{t|t,\tth}\right)^\top A^\top C^\top L^\top \left( \tp_t - \tc_t^\top Q \tc_t  \right) L C A \left(x_{t}-\hat{x}_{t|t,\tth}\right)  \Big | \hat{x}_{t|t-1}, y_t \right] \bigg \}, \\
    R_{11}& \!=\!\! \sum_{t=1}^{T} \bigg \{ 2 \mathbb{E}\left[z_{t+1}^\top L^\top\!\! \left( \tp_t \!-\! \tc_t^\top Q \tc_t  \right) (\tl_t \!-\! L) z_{t+1} \right] \!+\!\mathbb{E}\left[z_{t+1}^\top (\tl_t \!-\! L)^\top \!\! \left( \tp_t \!-\! \tc_t^\top Q \tc_t  \right) (\tl_t \!-\! L) z_{t+1}  \right]  \bigg \}, 
\end{align*}
where (\ref{useoptimism}) follows due to optimistic choice of system parameters. This gives us the following regret decomposition for the adaptive control period of \Alg: 
\begin{equation}
    \reg(T) \leq R_1 + R_2 - R_3 - R_4 - R_5 - R_6 - R_7 - R_8 - R_9 - R_{10} \!-\! R_{11} - R_{\text{update}}.
\end{equation}

\section{Regret Analysis, Proof of Theorem \ref{adaptive control regret}}
\label{SupSec:RegAnaly}
Notice that $R_1 - R_{11}$ given above have the same properties of $R_1 - R_{11}$ of \citet{lale2020regret}. The only difference is that during the adaptive control period of \Alg, the agent updates its estimate of the underlying system using the doubling trick mentioned in the main text and in Algorithm \ref{algo}. Therefore, with probability at least $1-5\delta$, the regret of each term has the following structure, 
\begin{align*}
    R_{i} = \tilde{\OO} \left( \frac{\Tw}{\sqrt{\Tw}} + \frac{2\Tw}{\sqrt{2\Tw}} + \frac{4\Tw}{\sqrt{4\Tw}} + \ldots \right)
\end{align*}
for $i=1, 3 \ldots, 11$. Since $R_2 = \tilde{\OO}(\sqrt{T-\Tw})$ and using Lemma \ref{doublingtrick}, we get that $R_i = \tilde{\OO} \left(\sqrt{T}\right)$ for $i=1, \ldots, 11$ with probability at least $1-5\delta$. Notice that there are $\log(T)$ policy changes, \textit{i.e.} there are $\log(T)$ terms in the summation of  $R_{\text{update}}$. Each term is bounded by $2\left(D\!+\!\| Q\| \left( \| C\|\!+\!\Delta C \right)^2  \right) \tilde{\mathcal{X}}^2 $. Thus, we have $|R_2| \leq 2\left(D\!+\!\| Q\| \left( \| C\|\!+\!\Delta C \right)^2  \right)\tilde{\mathcal{X}}^2 \log(T) = \OO (\log(T))$. Combining all, we conclude that during the adaptive control period of \Alg $\reg(T) = \tilde{\OO} \left( \sqrt{T} \right)$. 

\null\hfill$\square$

\section{System Identification with Non-Steady State Initial Point}\label{SupSec:NonSteady}
\begin{align}
    x_{t+1}&=\Bar{A_t} x_t+B u_t+A L_t y_t \nonumber \\
    y_{t}&=C x_t+e_t .
\end{align}
where $\Bar{A_t} = A - A L_t C$. If the system is at steady state, \textit{i.e.} $L_t = L = \sig  C^\top \left( C \sig  C^\top + \sigma_z^2 I \right)^{-1} $. Since the system is stable, the dynamics of the system approaches exponentially fast to the steady state dynamics. Therefore, starting at $x_0 = 0$ and with a long enough burning period such that $\|F_t - F\| = O\left(\frac{1}{\text{poly}(T)}\right)$, starting from arbitrary point will provide additional bias term in the estimation which decays over time: 
\begin{align}
  y_{H} &= \mathbf{M} \phi_{H} + e_{H} + (\mathbf{M_{H}} - \mathbf{M})\phi_{H} \nonumber \\
  y_{H+1} &= \mathbf{M} \phi_{H+1} + e_{H+1} + (\mathbf{M_{H+1}} - \mathbf{M})\phi_{H+1} + C \left( \prod_{i=1}^{H} \Bar{A}_{H+1-i} \right) x_1 \nonumber \\
  \vdots \nonumber \\
  y_{t} &= \mathbf{M} \phi_{t} + e_{t} + (\mathbf{M_{t}} - \mathbf{M})\phi_{t} + C \left( \prod_{i=1}^{H} \Bar{A}_{t-i} \right) x_{t-H} \nonumber
\end{align}
where 
\[
\mathbf{M} = \left[CF, \enskip C\Bar{A}F, \enskip \ldots, \enskip C \Bar{A}^{H-1} F, \enskip CB, \enskip C\Bar{A}B, \enskip \ldots, \enskip C\Bar{A}^{H-1} B \right] \in \mathbb{R}^{m \times (m+p)H}
\]
\[
\mathbf{M_t} \!=\! \left[CF_{t-1}, \enskip C\Bar{A}_{t-1}F_{t-2}, \enskip \ldots, \enskip C \left( \prod_{i=1}^{H-1} \Bar{A}_{t-i} \right) F_{t-H},\enskip CB,\enskip C\Bar{A}_{t-1}B,\enskip \ldots,\enskip C\left( \prod_{i=1}^{H-1} \Bar{A}_{t-i} \right) B \right]
\]
\[
\phi_t = \left[ y_{t-1}^\top \ldots y_{t-H}^\top \enskip u_{t-1}^\top \ldots \enskip u_{t-H}^\top \right]^\top \in \mathbb{R}^{(m+p)H}
\] 

Note that for any $t$, $\mathbf{M_{t}} - \mathbf{M}$ represents the model mismatch from the steady-state model parameters and the parameters of the evolving system. The noise terms are zero-mean including the effect of initial state since we assume that $x_0 = 0$. The model mismatch combined with the upper bound on $\phi_t$ can be used to define the additional bias in the estimation. Notice that this bias will decrease over time since the system approaches exponentially fast to the steady state dynamics. We leave the exact analysis to future work.  

\section{Technical Lemmas and Theorems}
\label{Technical}

\begin{theorem}[Matrix Azuma \citep{tropp2012user}]\label{azuma} Consider a finite adapted sequence $\left\{\boldsymbol{X}_{k}\right\}$ of self-adjoint matrices
in dimension $d,$ and a fixed sequence $\left\{\boldsymbol{A}_{k}\right\}$ of self-adjoint matrices that satisfy
\begin{equation*}
\mathbb{E}_{k-1} \boldsymbol{X}_{k}=\mathbf{0} \text { and } \boldsymbol{A}_{k}^{2} \succeq \boldsymbol{X}_{k}^{2}  \text { almost surely. }
\end{equation*}
Compute the variance parameter
\begin{equation*}
\sigma^{2}\coloneqq \left\|\sum_{k} \boldsymbol{A}_{k}^{2} \right\|
\end{equation*}
Then, for all $t \geq 0$
\begin{equation*}
\mathbb{P}\left\{\lambda_{\max }\left(\sum_{k} \boldsymbol{X}_{k} \right) \geq t\right\} \leq d \cdot \mathrm{e}^{-t^{2} / 8 \sigma^{2}}
\end{equation*}

\end{theorem}

\begin{theorem}[Self-normalized bound for vector-valued martingales~\citep{abbasi2011improved}]
\label{selfnormalized}
Let $\left(\mathcal{F}_{t} ; k \geq\right.$
$0)$ be a filtration, $\left(m_{k} ; k \geq 0\right)$ be an $\mathbb{R}^{d}$-valued stochastic process adapted to $\left(\mathcal{F}_{k}\right),\left(\eta_{k} ; k \geq 1\right)$
be a real-valued martingale difference process adapted to $\left(\mathcal{F}_{k}\right) .$ Assume that $\eta_{k}$ is conditionally sub-Gaussian with constant $R$. Consider the martingale
\begin{equation*}
S_{t}=\sum_{k=1}^{t} \eta_{k} m_{k-1}
\end{equation*}
and the matrix-valued processes
\begin{equation*}
V_{t}=\sum_{k=1}^{t} m_{k-1} m_{k-1}^{\top}, \quad \overline{V}_{t}=V+V_{t}, \quad t \geq 0
\end{equation*}
Then for any $0<\delta<1$, with probability $1-\delta$
\begin{equation*}
\forall t \geq 0, \quad\left\|S_{t}\right\|^2_{V_{t}^{-1}} \leq 2 R^{2} \log \left(\frac{\operatorname{det}\left(\overline{V}_{t}\right)^{1 / 2} \operatorname{det}(V)^{-1 / 2}}{\delta}\right)
\end{equation*}

\end{theorem}

\begin{lemma}[Norm of a subgaussian vector \citep{abbasi2011regret}]\label{subgauss lemma}
Let $v\in \R^d$ be a entry-wise $R$-subgaussian random variable. Then with probability $1-\delta$, $\|v\| \leq R\sqrt{2d\log(2d/\delta)}$.
\end{lemma}

\begin{lemma}[Doubling Trick \citep{jaksch2010near}]\label{doublingtrick}
For any sequence of numbers $z_{1}, \ldots, z_{n}$ with $0 \leq z_{k} \leq Z_{k-1} \coloneqq \max \left\{1, \sum_{i=1}^{k-1} z_{i}\right\}$
\begin{equation*}
   \sum_{k=1}^{n} \frac{z_{k}}{\sqrt{Z_{k-1}}} \leq(\sqrt{2}+1) \sqrt{Z_{n}}
\end{equation*}
\end{lemma}

\end{document}